\documentclass{article}
\usepackage[utf8]{inputenc}
\usepackage{hyperref}
\usepackage{listings, color}
\usepackage{tabularx}
\usepackage{booktabs} 
\usepackage{bm}

\usepackage{tabularx}
\usepackage{caption}
\usepackage{wrapfig}
\usepackage{adjustbox}
\usepackage{multirow}

\usepackage{mwe}
\usepackage[]{graphicx}
\usepackage{caption}
\usepackage{tikz}
\usepackage{tikz-cd}

\usepackage[T1]{fontenc}    %
\usepackage{hyperref}
\usepackage{url}            %
\usepackage{booktabs}       %
\usepackage{amsfonts}       %
\usepackage{nicefrac}       %
\usepackage{microtype}      %

\usepackage{graphicx}
\usepackage{subfig}
\usepackage{booktabs} %
\usepackage{caption}

\usepackage{framed}
\usepackage{amssymb}
\usepackage{amsfonts}
\usepackage{mathrsfs}
\usepackage{changepage}

\usepackage{mathtools}
\usepackage{array}
\usepackage{amsthm}
\usepackage{verbatim} 
\usepackage{enumerate}
\usepackage{bbm}
\usepackage{commath}
\usepackage{wrapfig}
\usepackage{amsbsy}
\usepackage{float}
\usepackage{amsmath}
\usepackage{algorithm}
\usepackage{tabularx}

\usepackage{enumitem}
\usepackage{lipsum}

\usepackage{listings}
\usepackage{xcolor}
\definecolor{codegreen}{rgb}{0,0.6,0}
\definecolor{codegray}{rgb}{0.5,0.5,0.5}
\definecolor{codepurple}{rgb}{0.58,0,0.82}
\definecolor{backcolour}{rgb}{0.95,0.95,0.92}
\lstdefinestyle{mystyle}{
    backgroundcolor=\color{white},
    commentstyle=\color{codegreen},
    keywordstyle=\color{magenta},
    numberstyle=\tiny\color{codegray},
    stringstyle=\color{codepurple},
    basicstyle=\ttfamily\footnotesize,
    breakatwhitespace=false,         
    breaklines=true,                 
    captionpos=b,                    
    keepspaces=true,                 
    numbers=left,                    
    numbersep=5pt,                  
    showspaces=false,                
    showstringspaces=false,
    showtabs=false,                  
    tabsize=2
}

\usepackage{ssArxiv}

\usepackage{amsmath,amsfonts,bm}
\usepackage{amsthm}
\usepackage{array, multirow}
\usepackage{adjustbox} 
\usepackage{thmtools}
\usepackage{thm-restate}
\DeclareMathOperator*{\argmin}{arg\,min}
\newtheorem{prop}{Proposition}

\newtheorem{cor}{Corollary}

\newtheorem{thm}[prop]{Theorem}

\newtheorem{principle}{Principle}

\usepackage{algorithm}
\usepackage[noend]{algpseudocode}
\algdef{SE}[SUBALG]{Indent}{EndIndent}{}{\algorithmicend\ }%
\algtext*{Indent}
\algtext*{EndIndent}

\newcommand{\RNum}[1]{\uppercase\expandafter{\romannumeral #1\relax}}
\usepackage{wrapfig}

\newcommand{\samecolorfootnote}[1]{\textsuperscript{\textcolor{darkred}{#1}}}

\author{
  Sharut Gupta$^\dagger$\samecolorfootnote{\textsuperscript{*}}, Joshua Robinson$^\dagger$\samecolorfootnote{\textsuperscript{*}},
  Derek Lim$^\dagger$,
  Soledad Villar$^\ddagger$,
  Stefanie Jegelka$^\dagger$  \\
  $^\dagger$MIT CSAIL,  $^\ddagger$Johns Hopkins University\\
  \texttt{\{sharut, joshrob, dereklim, stefje\}@mit.edu} \\
  \texttt{svillar3@jhu.edu}\\
}

\newcommand{\care}{\textsc{Care}\xspace}
\newcommand{\carefullform}{\textbf{C}ontrastive \textbf{A}ugmentation-induced \textbf{R}otational \textbf{E}quivariance}

\usepackage[utf8]{inputenc}
\usepackage[english]{babel}
\usepackage[authoryear]{natbib}

\title{Structuring Representation Geometry with \\ Rotationally Equivariant Contrastive Learning}

\begin{document}

\maketitle

\begin{abstract}
Self-supervised learning converts raw perceptual data such as images to a compact space where simple Euclidean distances measure meaningful variations in data. In this paper, we extend this formulation by adding additional geometric structure to the embedding space by enforcing transformations of input space to correspond to simple (i.e., linear) transformations of embedding space. Specifically, in the contrastive learning setting, we introduce an \emph{equivariance} objesctive and theoretically prove that its minima forces augmentations on input space to correspond to \emph{rotations} on the spherical embedding space. We show that merely combining our equivariant loss with a non-collapse term results in non-trivial  representations, without requiring invariance to data augmentations. Optimal performance is achieved by also encouraging approximate invariance, where input augmentations correspond to small rotations. Our method, \care: \textbf{C}ontrastive \textbf{A}ugmentation-induced \textbf{R}otational \textbf{E}quivariance, leads to improved performance on downstream tasks, and ensures sensitivity in embedding space to important variations in data (e.g., color) that standard contrastive methods do not achieve. Code is available at \href{https://github.com/Sharut/CARE}{\texttt{https://github.com/Sharut/CARE}}.
\end{abstract}

\footnotetext{Equal contribution.}

\section{Introduction}
It is only partially understood what structure neural network representation spaces should possess in order to enable intelligent behavior to efficiently emerge  \citep{ma2022principles}. One known key ingredient is to learn low-dimensional spaces in which simple Euclidean distances effectively measure the similarity between data.  A standout success of recent years has been the development of powerful methods for achieving this at web-scale using self-supervision \citep{simclr,schneider2021wav2vec,radford2021learning}. However, many use cases require the use of richer structural relationships that similarities between data cannot capture. One example that has enjoyed considerable success is the encoding of relations between objects (\emph{X is a parent of Y}, \emph{A is a treatment for B}) as simple transformations of embeddings (e.g., translations), which has driven learning with knowledge graphs \citep{bordes2013translating,sun2019rotate,yasunaga2022deep}. But similar capabilities have been notably absent from existing self-supervised learning recipes.

Recent contrastive self-supervised learning approaches have explored ways to close this gap by ensuring representation spaces are sensitive to certain transformations of input data (e.g., variations in color) \citep{dangovski2021equivariant,devillers2022equimod,garrido2023self,bhardwaj2023steerable}. Encouraging sensitivity is especially important in contrastive learning, as it is known to learn shortcuts that forget features that are not needed to solve the pretraining task \citep{robinson2021can}. 
This line of work formalizes sensitivity in terms of \emph{equivariance}: transformations of input data correspond to predictable transformations in representation space. Equivariance requires specifying a family of transformations $a \in \mathcal A$ in the input space, a corresponding transformation $T_a$ in representation space and training $f$ so that $f(a(x)) \approx T_a f(x)$. A typical choice of $T_a$ is a learnable feed-forward network, which acts non-linearly on embeddings \citep{devillers2022equimod,garrido2023self}. This approach has the disadvantage of encoding the relation between the embeddings of $x$ and $a(x)$ in a complex and hard to interpret manner. It also suffers from geometric pathologies, such as inconsistency under compositions: $T_{a_2 \circ a_1}f(x) \neq T_{a_2} T_{a_1}f(x)$.

To address these concerns we propose \care, an equivariant contrastive learning framework that learns to translate augmentations  in the input space (such as cropping, blurring, and jittering) into simple \emph{linear} transformations in feature space. Here, we use the sphere as our feature space (the standard space for contrastive learning), so we specifically consider transformations that are isometries of the sphere: rotations and reflections, i.e., orthogonal transformations. As orthogonal transformations are (intentionally) less expressive than prior non-linear formulations, our learning problem is more constrained and prior approaches for learning non-linear transforms do not apply (see Section \ref{sec: main method}). \care trains  $f$ to preserve angles, i.e., $f(a(x))^\top f(a(x')) \approx f(x)^\top f(x')$, a property that must hold if $f$ is orthogonally equivariant. We show that achieving low error on this seemingly weaker property also implies approximate equivariance and enjoys consistency under compositions. Critically,  we can easily integrate \care into contrastive learning workflows since both operate by comparing pairs of data. 

The key contributions of this work include: 

\begin{itemize}
    \item Introducing \care, a novel equivariant contrastive learning framework that trains transformations (cropping, jittering, blurring, etc.) in input space to approximately correspond to local orthogonal transformations in representation space. 
    \item Theoretically proving and empirically demonstrating that \care places an orthogonally equivariant structure on the embedding space.     \item Showing that \care increases sensitivity to features (e.g., color) compared to invariance-based contrastive methods, and  also improves performance on image recognition tasks.
\end{itemize}

\begin{figure}[t]
     \centering
     \includegraphics[width = \textwidth]{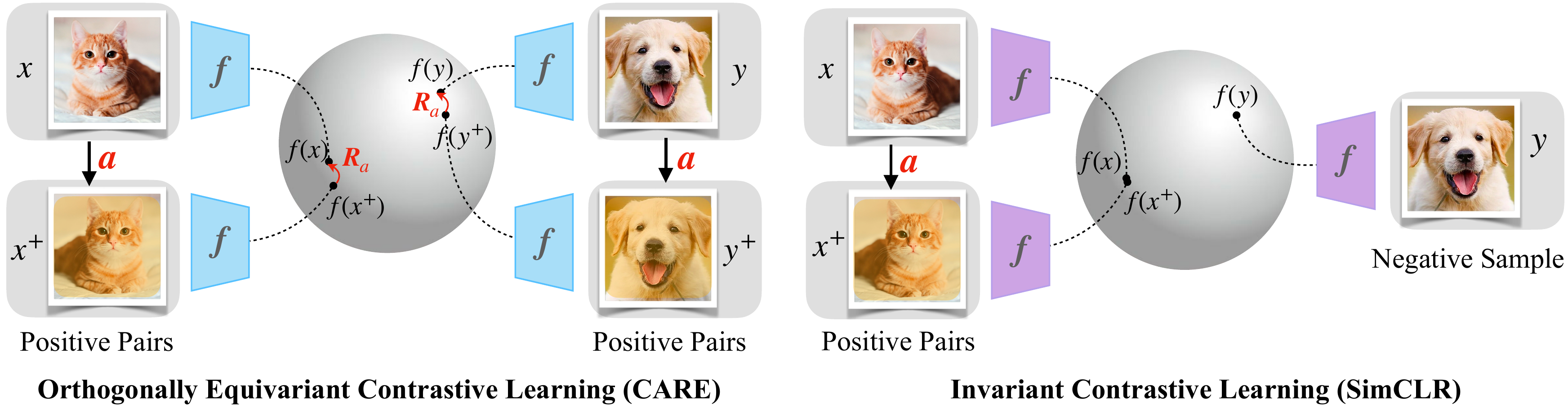}
     \caption{\care is an equivariant contrastive learning approach that trains augmentations (cropping, blurring, etc.) of input data to correspond to orthogonal transformations of embedding space.}
     \label{fig:main_figure}
\end{figure}

\section{Rethinking how augmentations are used in self supervised learning}\label{sec: inv and equi in SSL}

Given access only to samples from a marginal distribution $p(x)$ on some input space 
$\mathcal X$ such as images, the goal of representation learning is commonly to train a feature extracting model $f : \mathcal X \rightarrow \mathbb{S}^{d-1}$ mapping to the unit sphere $\mathbb{S}^{d-1} = \{ z \in \mathbb{R}^d: \| z \|_2 = 1\}$. A common strategy to automatically generate supervision from the data is to additionally introduce a space of augmentations $\mathcal A$, containing maps $a:\mathcal X \rightarrow \mathcal X$ which slightly perturb inputs $\bar{x}$ (blurring, cropping, jittering, etc.). Siamese self-supervised methods learn representation spaces that reflect the relationship between the embeddings of $x = a(\bar{x})$ and $x^+ = a^+(\bar{x})$, commonly by training $f$ to be invariant or equivariant to the augmentations in the input space \citep{chen2021exploring}. 

\paragraph{Invariance to augmentation.} One approach is to train $f$ to embed $x$ and $x^+$ nearby---i.e., so that $f(x) = f(x^+)$ is \emph{invariant} to augmentations. The InfoNCE loss \citep{oord2018representation,gutmann2010noise} used in contrastive learning achieves precisely this:
\begin{equation}\label{eqn: InfoNCE loss}
\mathcal L_{\text{InfoNCE}}(f) = \mathbb{E}_{x,x^+,\{x_i^-\}_{i=1}^N} \bigg [ - \log \frac{e^{f(x)^\top f(x^+) /\tau}}{e^{f(x)^\top f(x^+) / \tau} + \sum_{i=1}^N e^{f(x)^\top f(x_i^-) / \tau}}\bigg ],  
\end{equation} 
where $\tau>0$  is a temperature hyperparameter, and $x^-_i \sim p$ are negative samples from the marginal distribution on $\mathcal X$. As noted by \cite{wang2020understanding}, the contrastive training mechanism balances  invariance to augmentations with a competing objective: uniformly distributing embeddings over the sphere, which rules out trivial solutions such as constant functions. 

Whilst contrastive learning has produced considerable advances in large-scale learning \citep{radford2021learning}, several lines of work have begun to probe the fundamental role of invariance in contrastive learning. Two key conclusions of recent investigations include: 1) invariance limits the expressive power of features learned by $f$, as it removes information about features or transformations that may be relevant in fine-grained tasks \citep{lee2021improving,xie2022should}, and 2) contrastive learning actually benefits from not having exact invariance. For instance, a  critical role of the projection head is to expand the feature space so that $f$ is not fully invariant \citep{jing2021understanding}, suggesting that it is preferable for the embeddings of $x$ and $x^{+}$ to be close, but not identical.
\paragraph{Equivariance to augmentation.} To address the limitations of invariance, recent work has additionally proposed to control \emph{equivariance} (i.e., sensitivity) of $f$ to data transformations \citep{dangovski2021equivariant,devillers2022equimod,garrido2023self}. Prior works can broadly be viewed as training a set of features $f$ (sometimes alongside the usual invariant features) so that $f(a(x)) \approx T_a f(x)$ for samples $x \sim p$ from the data distribution where $T_a$ is some transformation of the embedding space. A common choice is to take $T_af(x) = \text{MLP}(f(x), a)$, a learnable feed-forward network, and optimize a loss $\| \text{MLP}(f(x), a) - f(a(x))\|_2$. Whilst a learnable MLP ensures that information about $a$ is encoded into the embedding of $a(x)$, it permits complex non-linear relations between embeddings and hence does not necessarily encode relations in a linearly separable way. Furthermore, it does not enjoy the beneficial properties of equivariance in the formal group-theoretic sense, such as consistency under compositions in general: $T_{a_2 \circ a_1}f(x) \neq T_{a_2} T_{a_1}f(x)$. 

Instead, this work introduces \care, an equivariant contrastive learning approach respecting two key design principles:
\begin{principle}
  The map $T_a$ satisfying  $f(a(x)) = T_a f(x)$ should be linear.
\end{principle}
\begin{principle}
 Equivariance should be learned from \emph{pairs} of data, as in invariant contrastive learning.
\end{principle}
The first principle asks that $f$ converts complex perturbations $a$ of input data into much simpler (i.e., linear) transformations in embedding space.  Specifically, we constrain the complexity of $T_a$ by considering isometries of the sphere,  $O(d) = \{ Q \in \mathbb{R}^{d \times d} : QQ^T = Q^TQ = I\}$, containing all rotations and reflections. Throughout this paper we define $f(a(x)) = T_a f(x)$ for $T_a \in O(d)$ to be \emph{orthogonal equivariance}. This approach draws heavily from ideas in linear representation theory \citep{curtis1966representation,serre1977linear}, which studies how to convert abstract group structures into matrix spaces equipped with standard matrix multiplication as the group operation.

The second principle stipulates \emph{how} we want to learn orthogonal equivariance. Naively following previous non-linear approaches is challenging as our learning problem is more constrained, requiring learning a mapping $a \mapsto R_a$ to orthogonal matrices. Furthermore, for a single $(a,x)$ pair, the orthogonal matrix $R_a$ such that $f(a(x)) = R_a f(x)$ is not unique, making it hard to directly learn $R_a$. We sidestep these challenges by, instead of explicitly learning $R_a$, training $f$ so that an augmentation $a$ applied to two different inputs $x,x^+$ produces the same change in embedding space.

Our method, \care,  encodes data augmentations (cropping, blurring, jittering, etc.) as $O(d)$ transformations of embeddings using an equivariance-promoting objective function. \care can be viewed as an instance of \emph{symmetry regularization}, a term introduced by \cite{shakerinavastructuring}.

\section{\care: \carefullform}\label{sec: main method}

This section introduces a simple and practical approach for training a model $f : \mathcal X \rightarrow \mathbb{S}^{d-1}$ so that $f$ is orthogonally equivariant: i.e., a data augmentation $a \sim \mathcal A$ (cropping, blurring, jittering, etc.) applied to any input $x \in \mathcal{X}$ causes the embedding $f(x)$ to transformed by the same $R_a \in O(d)$ for all $x \in \mathcal{X}$: $f(a(x)) = R_a f(x)$.

To achieve this, we consider the following loss: 
\begin{equation}\mathcal L_\text{equi}(f) = \mathbb{E}_{a \sim \mathcal A} \mathbb{E}_{x,x' \sim \mathcal X} \big  [ f(a(x'))^\top f(a(x)) - f(x) ^\top f(x') \big ]^2\end{equation}

Since inner products describe angles on the sphere, this objective enforces the angles between the embeddings of independent samples $x$ and $x'$ to be the same as those between their transformed counterparts $a(x)$ and $a(x')$. This is necessarily true if $f$ is orthogonally equivariant or, more generally, $R_a \in O(d)$ exists. But the converse---that $\mathcal L_\text{equi}=0$ implies orthogonal equivariance---is non-obvious. In Section \ref{sec: theory} we theoretically analyze $\mathcal L_\text{equi}$, demonstrating that it does indeed enforce mapping input augmentations to orthogonal transformations of embeddings. In practice, we replace the $ f(x) ^\top f(x')$ term with $ f(a'(x)) ^\top f(a'(x'))$ for a freshly sampled $a' \sim \mathcal A$, noting that minimizing this variant also minimizes $\mathcal L_\text{equi}$, if we assume $a'$ can be the identity function with non-zero probability.
\begin{figure}[tbp]
    \centering
    \begin{minipage}{0.59\textwidth}
        \centering
        \includegraphics[width=\textwidth]{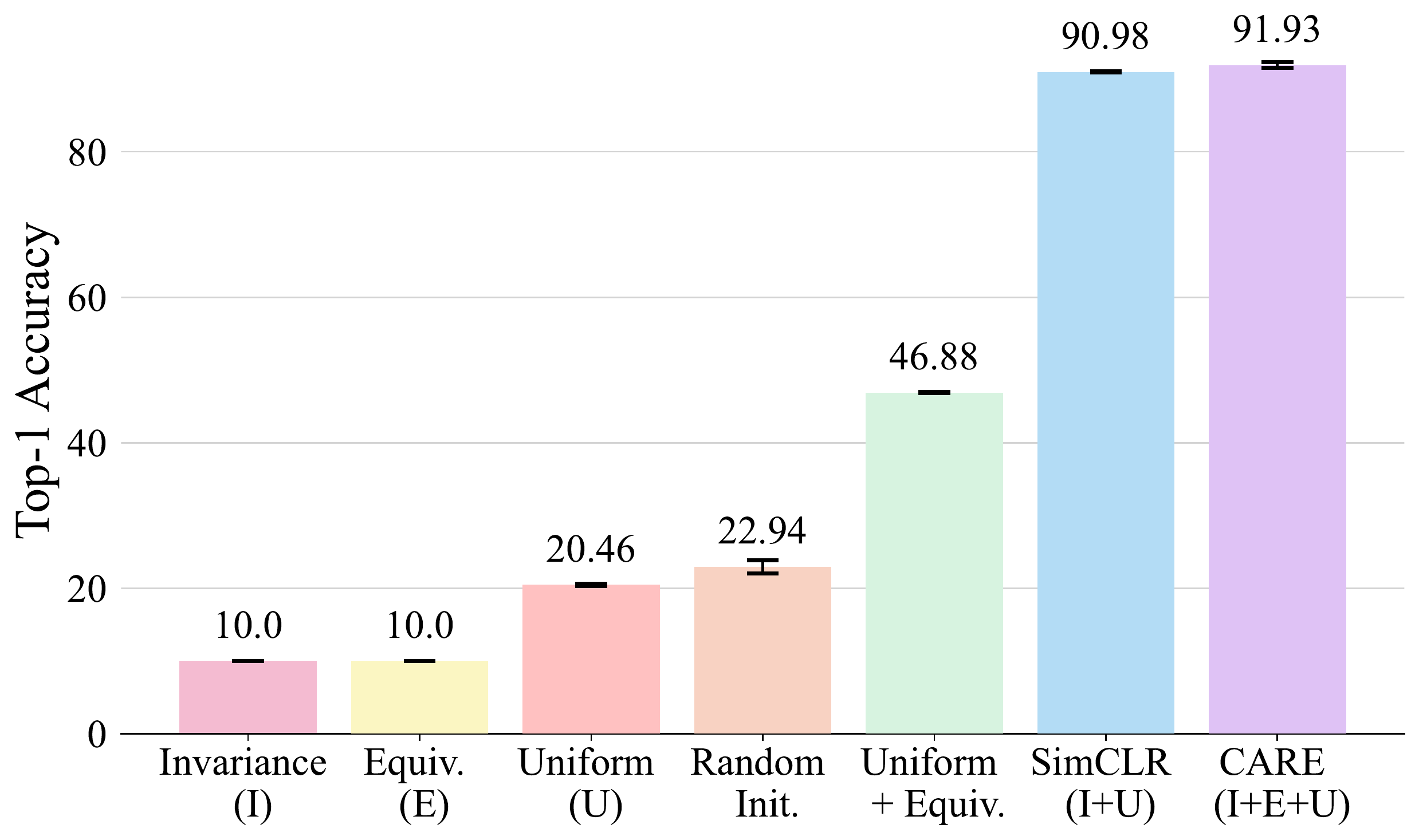}
        \caption{Ablating different loss terms. Combining $\mathcal L_\text{equi}$ with a uniformity promoting non-collapses term suffices to learn non-trivial features. However, optimal performance is achieved when encouraging \emph{smaller} rotations, as in \care. ResNet-50 models pretrained on CIFAR10 and evaluated with linear probes.}
        \label{fig: losses ablation}
    \end{minipage}\hfill
    \begin{minipage}{0.40\textwidth}
        \centering
        \includegraphics[width=\textwidth]{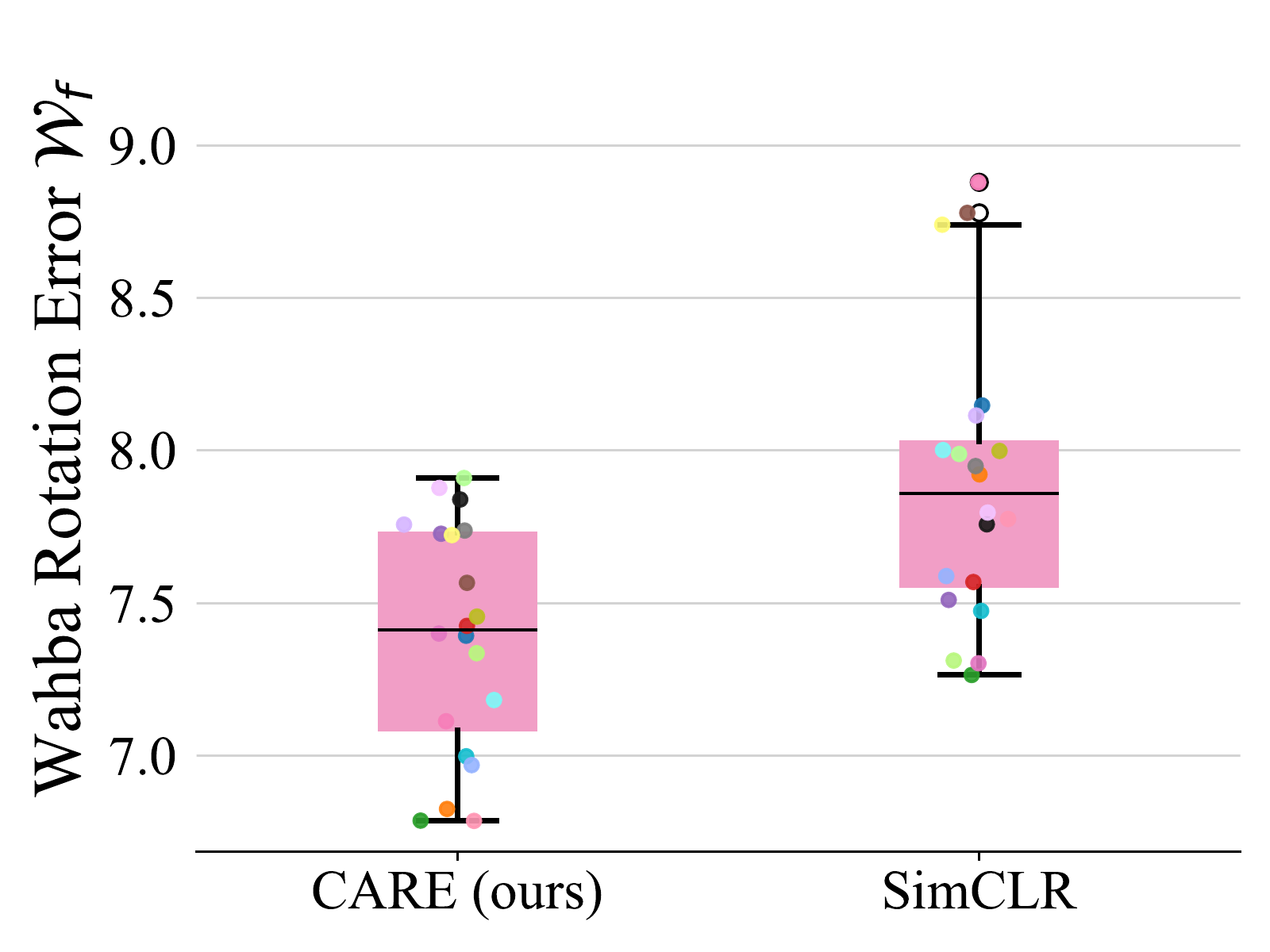}
        \caption{\care learns a representation space with better rotational equivariance. We compare the models by the error of optimally rotating a set of embeddings to match the embeddings of augmented inputs, known as Wahba's problem (Sec. \ref{sec: measuring rotations}).}
        \label{fig: wahba boxplot}
    \end{minipage}
\end{figure}
A trivial but undesirable solution that minimizes $\mathcal L_\text{equi}$ is to collapse the embeddings of all points to be the same 
(see Figure \ref{fig: losses ablation}). One natural approach to avoiding trivial solutions is to combine the equivariance loss with a non-collapse term such as the uniformity $\mathcal L_\text{unif}(f) = \log \mathbb{E}_{x,x'\sim \mathcal X} \exp \big ( f(x)^\top f(x') \big )$  \citep{wang2020understanding}
 whose optima $f$ distribute points uniformly over the sphere:
\begin{equation} \mathcal L(f) = \mathcal L_\text{equi}(f) + \mathcal L_\text{unif}(f).\end{equation}
This is directly comparable to the InfoNCE loss, which can similarly be decomposed into two terms:
 \begin{equation} \mathcal L_\text{InfoNCE}(f) = \mathcal L_\text{inv}(f) + \mathcal L_\text{unif}(f)\end{equation}
 where $\mathcal  L_\text{inv}(f) = \mathbb{E}_{a,a' \sim \mathcal A} \| f(a(x)) - f(a'(x))
\|$ is minimized when $f$ is invariant to $\mathcal A$---i.e., $f(a(x)) =f(x)$.
Figure \ref{fig: losses ablation} shows that training using $\mathcal L_\text{equi} + \mathcal L_\text{unif}$ yields non-trivial representations. However, the performance is below that of invariance-based contrastive learning approaches. We hypothesize that this is because data augmentations---which make small perceptual changes to data---should correspond to \emph{small} perturbations of embeddings, which $\mathcal L_\text{equi}$ does not enforce.

To rule out this possibility, we introduce \care: \carefullform. \care additionally enforces the orthogonal transformations in embedding space to be \emph{localized} by reintroducing an invariance loss term $\mathcal L_\text{inv}$ to encourage $f$ to be approximately invariant. Doing so breaks the indifference of $\mathcal L_\text{equi}$ between large and small rotations,  biasing towards small. Specifically, we propose the following objective that combines our equivariant loss with InfoNCE: 
\begin{equation} \mathcal L_{\text{\care}}(f) = \mathcal L_\text{inv}(f) + \mathcal L_\text{unif}(f) + \lambda \mathcal L_{\text{equi}}(f) \end{equation}
where $\lambda$ weights the equivariant loss. We note that many variations of this approach are possible. For instance, the equivariant loss and InfoNCE loss could use different augmentations, resulting in invariance to specific transformations while maintaining rotational equivariance to others, similar to \cite{dangovski2021equivariant}. The InfoNCE loss can also be replaced by other Siamese self-supervised losses. We leave further exploration of these possibilities to future work.
In all, \care consists of three components: (i) a term to induce orthogonal equivariance; (ii) a non-collapse term; and (iii) an invariance term to enforce localized transformations on the embedding space.

\subsection{Theoretical properties of the orthogonally equivariant loss}\label{sec: theory}
In this section, we establish that matching angles via $\mathcal L_\text{equi}$ leads to a seemingly stronger property. Specifically, $\mathcal L_\text{equi}=0$ implies the existence of an orthogonal matrix $R_a \in O(d)$ for any augmentation $a$, such that $f(a(x)) = R_a f(x)$ holds for all $x$. The converse also holds and is easy to see. Indeed, suppose such an $R_a \in O(d)$ exists. Then, $f(a(x'))^\top f(a(x)) = f(x')^\top R_a^\top R_a f(x) = f(x)^\top f(x')$, which implies $\mathcal L_\text{equi}(f) = 0$. We formulate the first direction as a proposition.

\begin{restatable}{prop}{rotationthm}
\label{prop: equivariance loss learns rotations}
    Suppose $ \mathcal L_\text{equi}(f) =0$. Then for almost every $a \in \mathcal A$, there is an orthogonal matrix $R_{a} \in O(d)$ such that $f(a(x)) = R_{a}f(x)$ for almost all $x \in \mathcal X$.
\end{restatable}
Figure \ref{fig:main_figure} illustrates this result. Crucially $R_{a}$ is independent of $x$, without which the Proposition \ref{prop: equivariance loss learns rotations} would be trivial. That is, a single orthogonal transformation $R_{a}$ captures the impact of applying $a$ across the entire input space $\mathcal{X}$. Consequently,  low $\mathcal L_\text{equi}$ loss converts ``unstructured'' augmentations in input space to have %
a structured geometric interpretation as rotations in the embedding space.

\par This result can be expressed as the existence of a mapping $\rho: \mathcal A \rightarrow O(d)$ that encodes the space of augmentations within $O(d)$. This raises a natural question: how much of the structure of $\mathcal A$ does this encoding preserve? For instance, assuming $\mathcal A$ is a semi-group (i.e., closed under compositions $a' \circ a \in \mathcal A$), does  this transformation respect compositions: $f(a'(a(x)) = R_{a'} R_{a} f(x)$? This property does not hold for non-linear actions \citep{devillers2022equimod}, but does for orthogonal equivariance:

\begin{restatable}{cor}{corrgroup}
\label{cor: homomorphism result}
    If  $ \mathcal L_\text{equi}(f) =0$, then $\rho : \mathcal A \rightarrow O(d)$ given by $\rho(a) = R_a$ satisfies $\rho(a' \circ a) = \rho(a')\rho(a)$ for almost all $a,a'$. That is, $\rho$ defines a group action on $\mathbb{S}^{d-1}$ up to a set of measure zero.
\end{restatable}

Formally, this result states that if $\mathcal A$ is a semi-group, then $\rho : 
\mathcal A \rightarrow O(d)$ defines a group homomorphism (or linear representation of $\mathcal A$ in the sense of representation theory \citep{curtis1966representation,serre1977linear}, a branch of mathematics that studies the encoding of abstract groups as spaces of linear maps). 
\begin{wrapfigure}{r}{0.4\textwidth}
    \centering
    \includegraphics[width=0.37\textwidth]{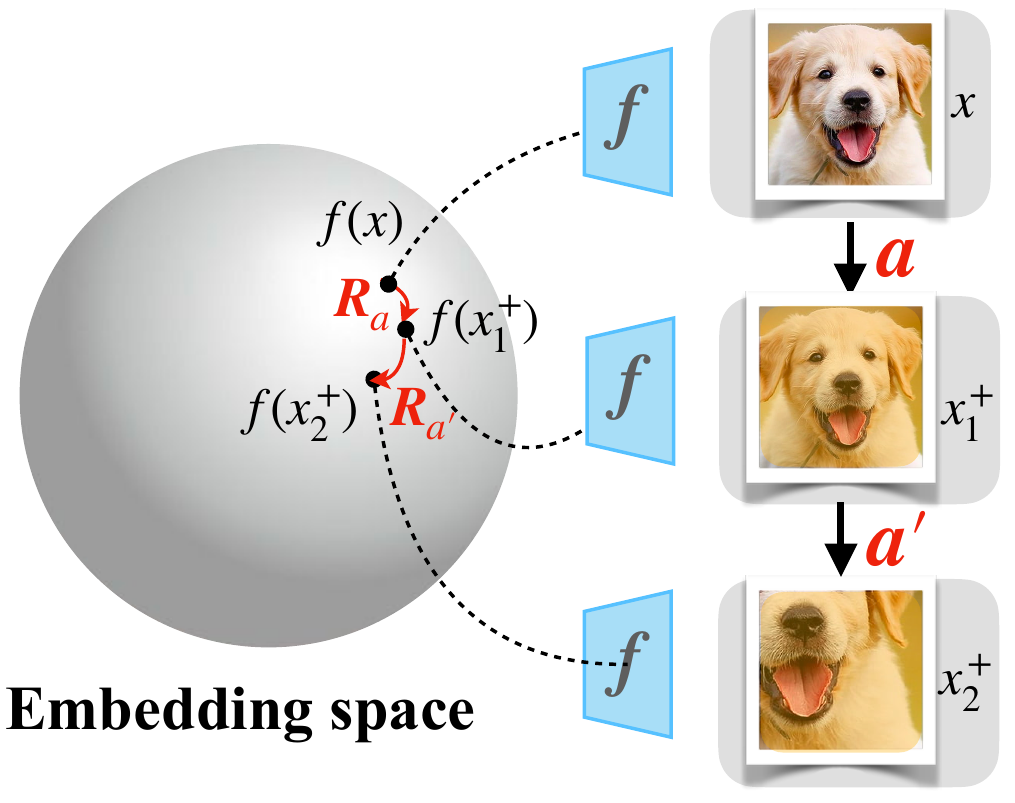}
    \caption{When $\mathcal L_\text{equi}=0$, compositions of augmentations correspond to compositions of rotations.}
    \label{fig:figure_label}
\end{wrapfigure}
To exactly attain $ \mathcal L_\text{equi}(f) =0$, the space of augmentations $\mathcal A$ needs to have a certain structure, but this becomes less restrictive if $d$ is large. Assuming for simplicity that $\mathcal A$ is a group, the first isomorphism theorem for groups states that $\rho(\mathcal A)\simeq \mathcal A/\operatorname{ker}(\rho)$. For instance, if $\operatorname{ker}(\rho)$ is trivial, the equivariant loss can be exactly zero when the group of augmentations is a subgroup of the orthogonal group. Examples include orthogonal transformations or rotations that fix a subspace---i.e., $O(d')$ or $SO(d')$ with $d'\leq d$---or subgroups of the permutation group on $d$ elements. Furthermore, the Peter-Weyl theorem implies that any compact Lie group can be realized as a closed subgroup of $O(d)$ for some $d$ \citep{peter1927vollstandigkeit}.  In practice, we are learning  equivariance, so do not expect to achieve exactly zero loss. Instead, the primary focus is on achieving better approximate equivariance (see Figure \ref{fig:relative_measure}), while enforcing small transformations that remain local.

\subsection{Extensions to other groups}\label{sec: extension to other groups}
Proposition \ref{prop: equivariance loss learns rotations} states that perfectly optimizing $\mathcal L_\text{equi}=0$ produces an $f$ that is equivariant, encoding augmentations in the input space as orthogonal transformation in the embedding space.
Notably, since the computation of $\mathcal L_\text{equi}$ solely relies on pairwise data instances $x,x'\in \mathcal X$, it naturally aligns with the contrastive learning paradigm that already works with pairs of data. However, this alignment does not hold in cases where orthogonal transformations in the embedding space are replaced by arbitrary group actions.

Mathematically, invariants of the action of $O(d)$ on $n$ points---seen in $(\mathbb R^d)^n$ as $Q\,(x_1, \ldots x_n)=(Q\,x_1, \ldots, Q\,x_n)$---can be expressed as a function of pairs of objects $(x_i^\top x_j)_{i,j=1\ldots n}$. 
This is because the orthogonal group is defined as the stabilizer of a bilinear form. In other words, letting $B(x,x') = x^\top x'$ denote the standard inner product, we have
\begin{equation}O(d) = \{ A \in GL(d): B(Ax,Ax') = B(x,x') \text{ for all } x,x'\in \mathbb R^d \}.\end{equation}
This argument applies more generally to other groups that are defined as stabilizers of bilinear forms. For instance, the Lorentz group, which has applications in the context of special relativity, can be defined as the stabilizer of the Minkowski inner product. Additionally, the symplectic group, which is used to characterize Hamiltonian dynamical systems, can be defined in a similar manner.

Such extensions to other groups allow us to use \care for different embedding space geometries. For instance, several recent works have used a hyperbolic space as an embedding space for self-supervised learners~\citep{ge2022hyperbolic, yue2023hyperbolic, desai2023hyperbolic}. If we constrain our embedding to a hyperboloid model of hyperbolic space, then linear isometries of this space are precisely the Lorentz group. Hence, using our equivariance loss with the Minkowski inner product replacing the Euclidean inner product would allow us to learn hyperbolic representations that transform the embeddings according to the action of the Lorentz group when an augmentation is applied to the input space. Further discussions on extensions to other groups are given in Appendix \ref{app: extension to other groups}.

\section{Measuring orthogonal action on embedding space}\label{sec: measuring rotations}
To probe the geometric properties of \care, we consider two efficiently computable metrics for empirically measuring the orthogonal equivariance in the embedding space. We report empirical results with these measures in Section~\ref{sec:wahba_empirical}.

\noindent \textbf{Wahba's problem.}
Proposition \ref{prop: equivariance loss learns rotations} states that a single orthogonal matrix $R_a \in O(d)$ describes the effect of augmentation $a$ for all input points $x$---i.e., $R_a$ does not depend on $x$. Hence, a natural way to assess the equivariance of $f$ is to sample a batch of data $\{x_i\}^{n}_{i=1}$ and an augmentation $a$ and test to what extent applying $a$ transforms the embeddings of each $x_i$ the same way. To measure this we compute a single rotation that approximates the map from $f(x_i)$ to $f(a(x_i))$ for all $i$. Let $F$ and $F_a \in \mathbb{R}^{d \times n}$ have $i$th columns $f(x_i)$ and $f(a(x_i))$ respectively, then we compute the error
\begin{equation}\label{eqn: Wahba}
   \mathcal  W_f = \min\nolimits_{R \in SO(d) } \| RF - F_a\|_\text{Fro},
\end{equation}
where  $\| \cdot \|_\text{Fro}$ denotes the Frobenius norm. If $\mathcal W_f=0$, then $f(a(x_i)) = R_{a}f(x_i)$  for all $i$.  Problem (\ref{eqn: Wahba}) is a well-studied problem known as \emph{Wahba's problem}. 
The analytic solution to Wahba's problem is easily computed. It is nearly
$R^* = UV^\top$ where $U\Sigma V^\top$ is a singular value decomposition of $F_a F^\top$. However, a slight modification is required as this $R^*$ could have determinant $\pm 1$, and therefore may not belong to $SO(d)$. Fortunately, the only modification needed is to re-scale so that the determinant is one: $R^* = U \cdot \text{diag}\big \{ \mathbf{1}_{(n-1)}, \text{det}(U)\text{det}(V) \big \} \cdot V^\top$ 
where $ \mathbf{1}_{n}$ denotes the vector in $\mathbb{R}^n$ of all ones. 
This method of computing the solution $R^*$ to Wahba's problem is known as Kabsch's algorithm \citep{kabsch1976solution}, and has been used for aligning point clouds to, e.g., compare molecular and protein structures and spacecraft attitude determination \citep{markley2014fundamentals,kneller1991superposition}. We use this algorithm to compute the optimal solution $R^*$ and further compare the error of interest as $\mathcal W_f = \| R^* F - F_a\|$.

\noindent \textbf{Relative rotational equivariance.}
Optimizing for the \care objective may potentially result in learning invariance rather than equivariance. Specifically, for input image $x$, $f(a(x))=f(x)$ for $a \in \mathcal{A}$ is a trivial optimal solution of $\argmin_{f}\mathcal L_\text{equi}(f)$. To check that our model is learning non-trivial equivariance, we consider a metric similar to one proposed by \cite{bhardwaj2023steerable} for measuring the equivariance \emph{relative} to the invariance of $f$: \begin{equation}\gamma_f = \mathbb{E}_{a \sim \mathcal A} \mathbb{E}_{x,x' \sim \mathcal X} \Bigg \{\frac{(\|f(a(x')) - f(a(x))\|^2 - \|f(x') - f(x)\|^2)^2}{(\|f(a(x')) - f(x')\|^2 + \|f(a(x)) - f(x)\|^2)^2} \Bigg \}.\end{equation}
Here, the denominator measures the invariance of the representation, with smaller values corresponding to greater invariance to the augmentations. The numerator, on the other hand, measures equivariance and can be simplified to $[ f(a(x'))^\top f(a(x)) - f(x) ^\top f(x') \big ]^2$ (i.e., $\mathcal L_\text{equi}(f)$) up to a constant, because $f$ maps to the unit sphere. The ratio $\gamma_f$ of these two terms measures the non-trivial equivariance, with a lower value implying greater non-trivial orthogonal equivariance.

\section{Experiments}\label{sec: experiments}
We examine the representations learned by \care, as well as those obtained from purely invariance-based contrastive approaches. We study three aspects of our model: 1) quantitative measures of orthogonal equivariance, 2) qualitative evaluation of the effect of equivariance on sensitivity to data transforms, and 3) performance of features learned by \care on image classification tasks.
We describe our experiment configurations in detail in Appendix \ref{app: experimental details}.

\subsection{Qualitative assessment of equivariance}
A key property promised by equivariant contrastive models is sensitivity to specific augmentations. 
To qualitatively evaluate the sensitivity, or equivariance, of our models, we consider an image retrieval task on the Flowers-102 dataset \citep{nilsback2008automated}, as considered by \cite{bhardwaj2023steerable}. Specifically, when presented with an input image $x$, we extract the top 5 nearest neighbors based on the Euclidean distance of $f(x)$ and $f(a(x))$, where $a \in \mathcal{A}$. We report the results of using color jitter as a transformation of the input, comparing the invariant (SimCLR) and our equivariant (\care) models in Figure \ref{fig:flowers_knn}. We see that retrieved results for the \care model exhibit greater variability in response to a change in query color compared to the SimCLR model. Notably, the color of the retrieved results for all queries in the SimCLR model remains largely invariant, thereby confirming its robustness to color changes.

\begin{figure}[tb!]
    \centering
    \includegraphics[width = \textwidth]{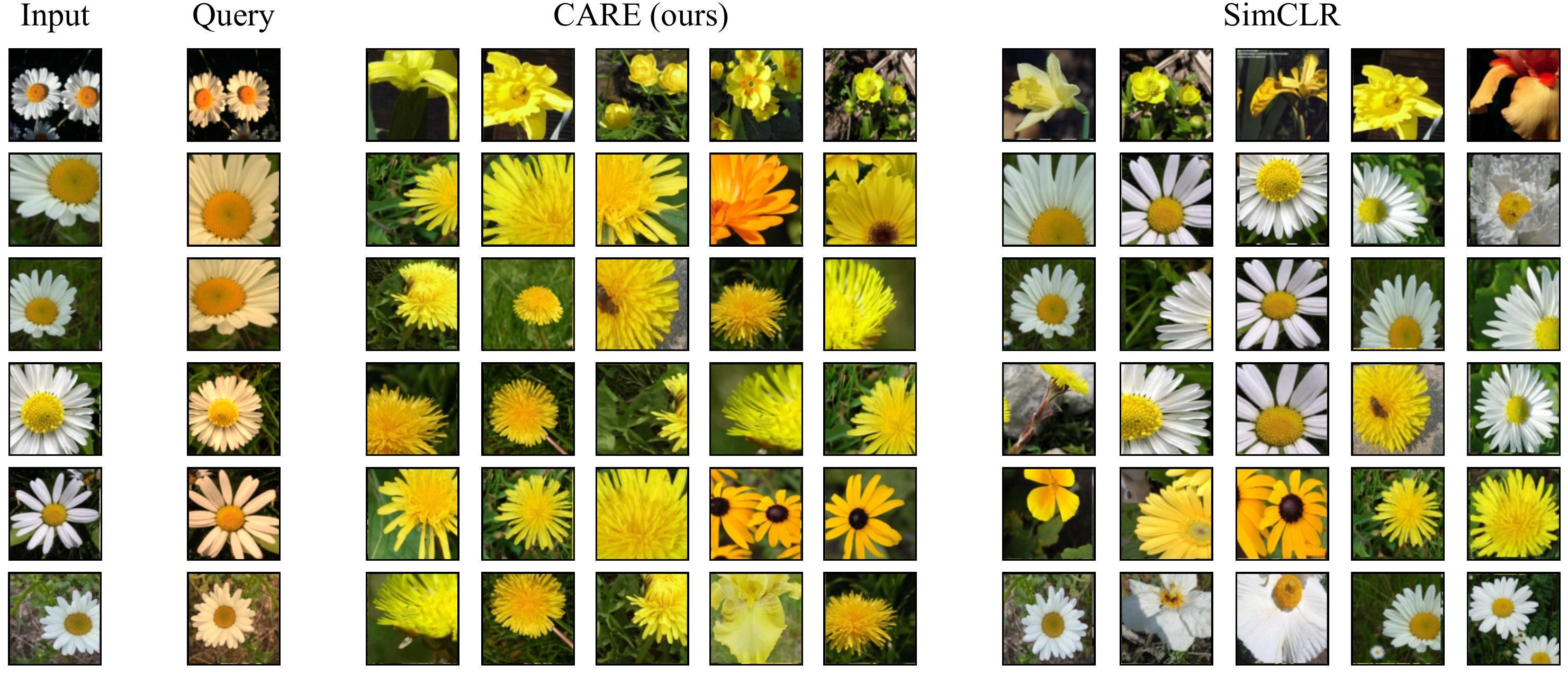}
    \caption{\care exhibits sensitivity to features that invariance-based contrastive methods (e.g., SimCLR) do not. For each input we apply color jitter to produce the query image. We then retrieve the $5$ nearest neighbors in the embedding space of \care and SimCLR.}
    \label{fig:flowers_knn}
\end{figure}

\subsection{Quantitative measures for orthogonal equivariance}\label{sec:wahba_empirical}

\textbf{Wahba's Problem} We compare ResNet-18 models pretrained with \care and with SimCLR on CIFAR10. For each model, we compute the optimal value $\mathcal W_f$ of Wahba's problem, as introduced in Section \ref{sec: measuring rotations}, over repeated trials. In each trial, we sample a single augmentation $a \sim \mathcal A$ at random and compute $\mathcal W_f$ for $f=f_\care$ and $f=f_\text{SimCLR}$ over the test data. We repeat this process 20 times and plot the results in Figure \ref{fig: wahba boxplot}, where the colors of dots indicate the sampled augmentation. Results show that \care has a lower average error and worst-case error. Furthermore, comparing point-wise for a single augmentation, \care achieves lower error in nearly all cases.

\begin{figure}[!htb]
\centering
\begin{minipage}{0.32\textwidth}
\centering
\includegraphics[width=\textwidth]{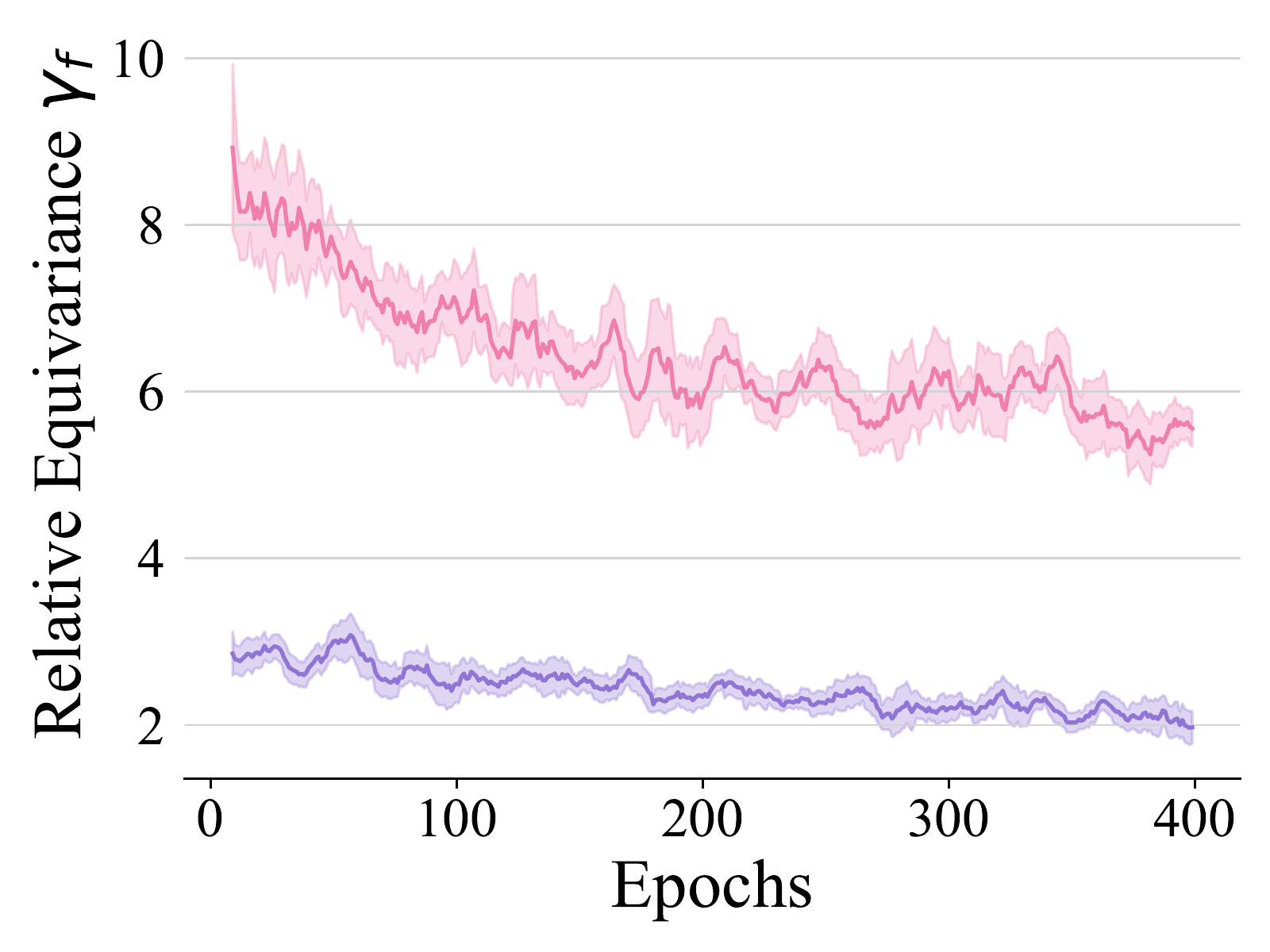}
\end{minipage}
\hfill
\begin{minipage}{0.32\textwidth}
\centering
\includegraphics[width=\textwidth]{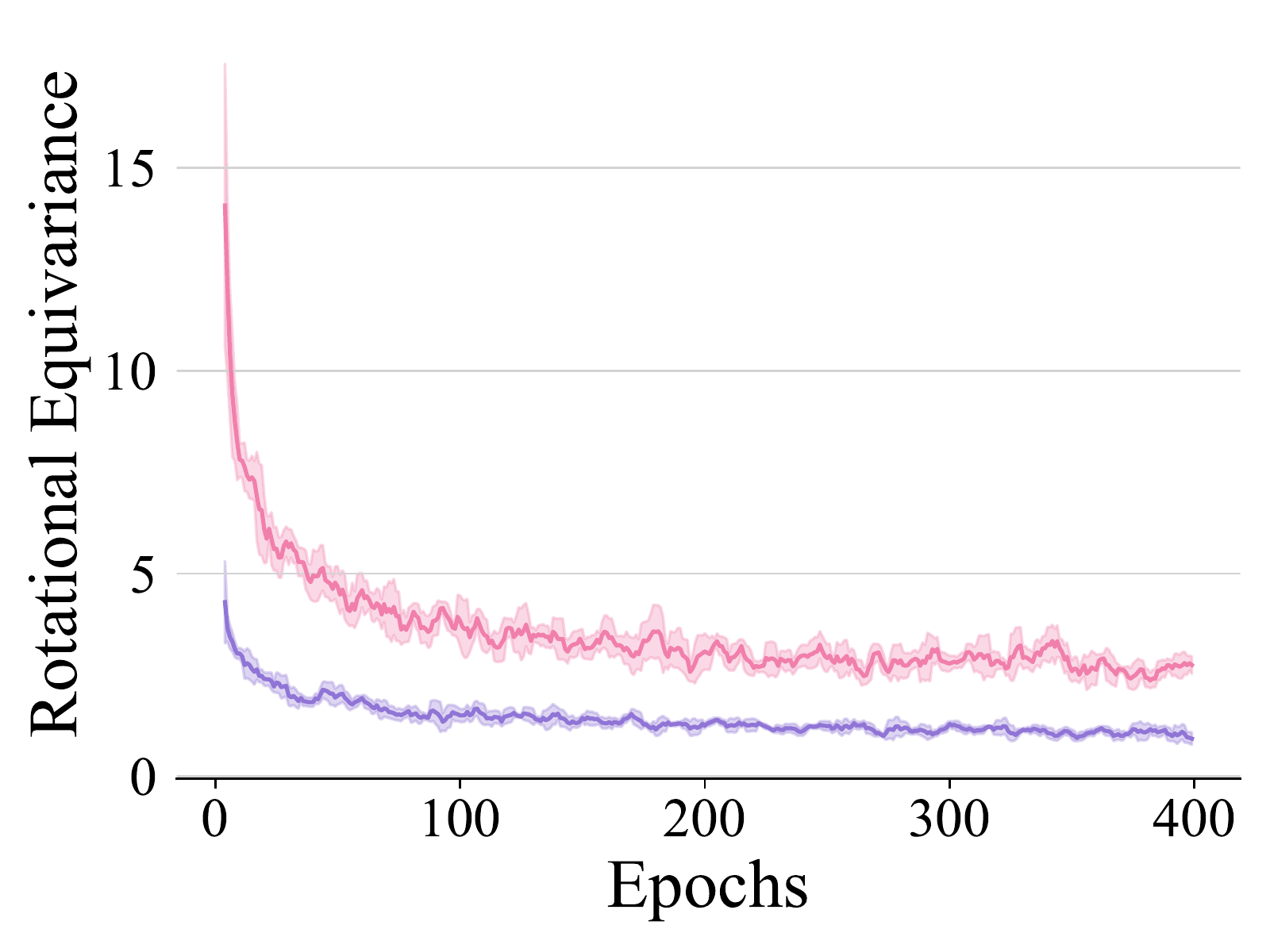}
\end{minipage}
\hfill
\begin{minipage}{0.32\textwidth}
\centering
\includegraphics[width=\textwidth]{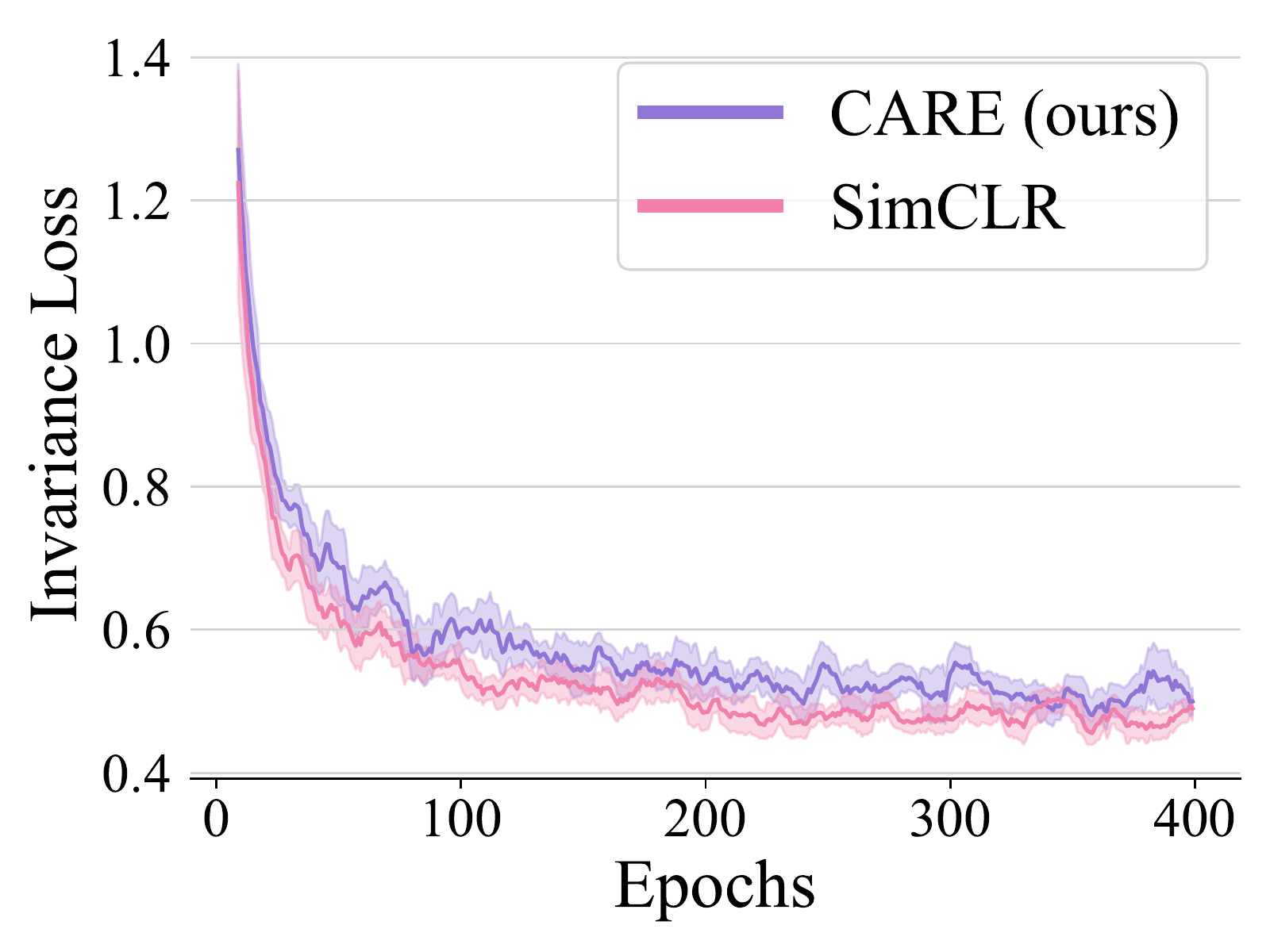}
\vspace{-15pt}
\end{minipage}

\caption{\textbf{Relative rotational equivariance} (lower is more equivariant). Both \care and invariance-based contrastive methods (e.g., SimCLR) produce \emph{approximately} invariant embeddings. However, they differ in their residual sensitivity to augmentations. \care learns a considerably more rotationally structured embedding space. We note that this is in part because \care is less invariant to augmentations (higher invariance loss).}\label{fig:relative_measure}
\end{figure}
\noindent \textbf{Relative rotational equivariance.} We measure the relative rotational equivariance for both \care and SimCLR over the course of pretraining by following the approach outlined in Section \ref{sec: measuring rotations}. Specifically, we compare ResNet-18 models trained using \care and SimCLR on CIFAR10. From Figure \ref{fig:relative_measure}, we observe that both the models produce embeddings with comparable non-zero invariance loss $\mathcal L_\text{inv}$, indicating approximate invariance. However, they differ in their sensitivity to augmentations, with \care attaining a much lower relative equivariance error. Importantly, this shows that \care is \emph{not} achieving lower equivariance error $\mathcal L_\text{equi}$ by collapsing to invariance, a trivial form of equivariance.

\begin{wrapfigure}{r}{0.47\textwidth}
    \centering
    \includegraphics[width=0.47\textwidth]{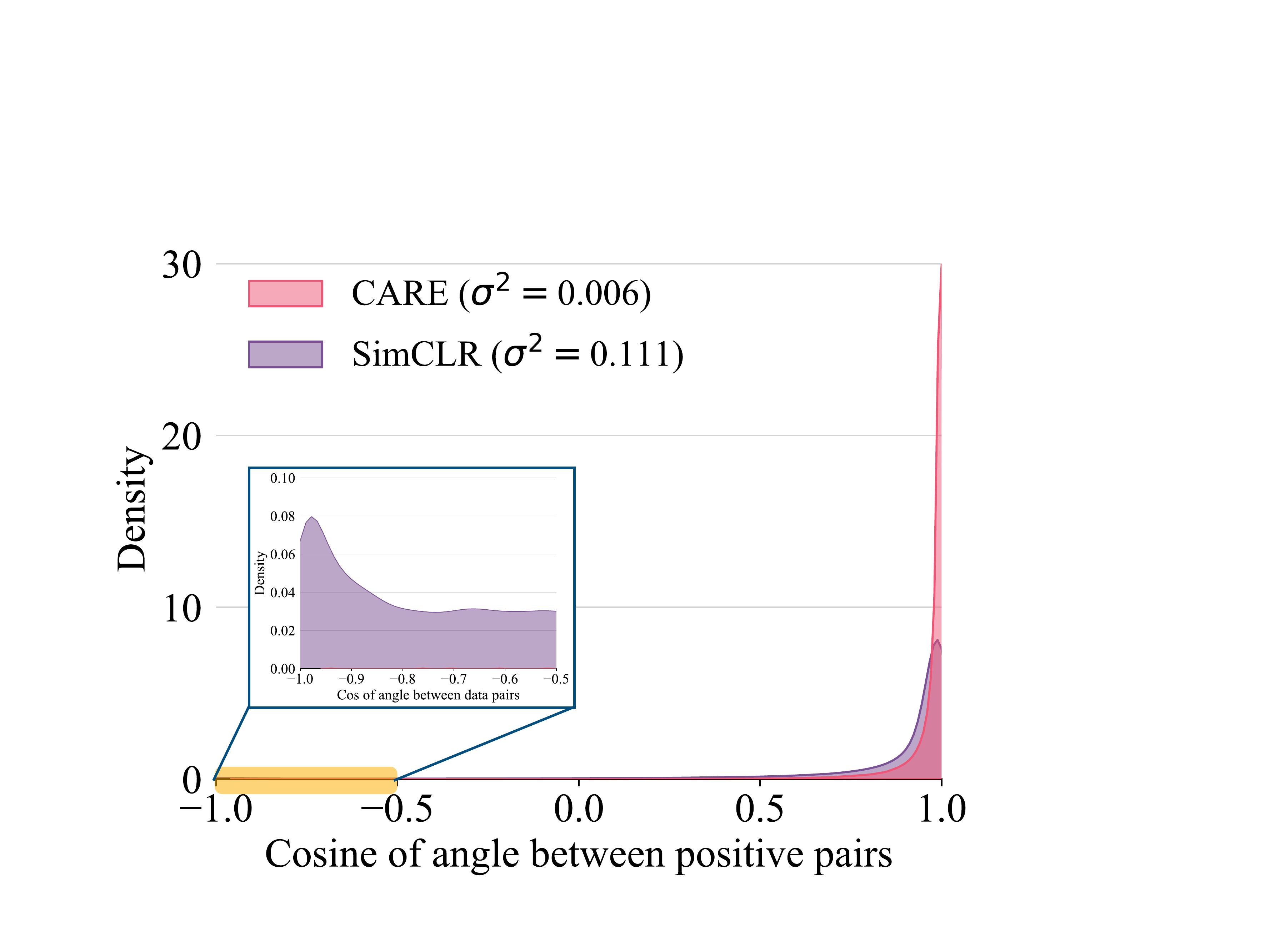}
    \caption{Histogram of the cosine of angles between data pairs for \care and SimCLR. \care exhibits a significantly lower variance of cosine similarity values compared to SimCLR.}
    \label{fig:angle_kde}
    \vspace{-2pt}
\end{wrapfigure}

\noindent \textbf{Analyzing structure on a 2D manifold.} 
To further study $\mathcal L_\text{equi}$, we train an encoder $f$ that projects the input onto $\mathbb{S}^1$, the unit circle in the 2D plane. In this case, orthogonal transformations are characterized by \emph{angles}. We sample an augmentation $a \sim \mathcal A$ and measure the cosine of the angle between pairs $f(x)$ and $f(a(x))$ for all $x$ in the test set. This process is repeated for 20 distinct sampled augmentations, and the density of all recorded cosine angles is recorded in Figure \ref{fig:angle_kde}. Both \care and SimCLR exhibit high density close to 1,
demonstrating approximate invariance.
However, unlike \care, SimCLR exhibits non-zero density in the region $-0.5$ to $-1.0$, indicating that the application of augmentations significantly displaces the embeddings. Additionally, \care consistently exhibits lower variance $\sigma^2$ of the cosine angles between $f(x)$ and $f(a(x))$ for a fixed augmentation, as expected given that it is supposed to transform all embeddings in the same way.

\subsection{Linear probe for image classification}
Next, we examine the quality of features learned by \care for solving image classification tasks. We train ResNet-50 models on four datasets: CIFAR10, CIFAR100, STL10, and ImageNet100 using \care and SimCLR (see Appendix \ref{app: experimental details} for details). To illustrate that \care can also be integrated into other self-supervised frameworks, we train MoCo-v2 models on ImageNet100 (with and without \care). We refer to the model trained using \care with SimCLR or MoCo-v2 backbone as $\text{\care}_{\text{SimCLR}}$ and $\text{\care}_{\text{MoCo-v2}}$ respectively.  For each method and dataset, we evaluate the quality of the learned features by training a linear classifier (i.e., probe \citep{alain2016understanding}) on the frozen features of $f$ and report the test set performances in Figure \ref{fig:accuracy_bar_plot}. In all cases, we run the linear probe training for five random seeds and report averages.  We find consistent improvements in performance using \care, showing the benefits of our structured embedding approach for image recognition tasks.

\begin{figure}[htb!]
    \centering
    \includegraphics[width = \textwidth]{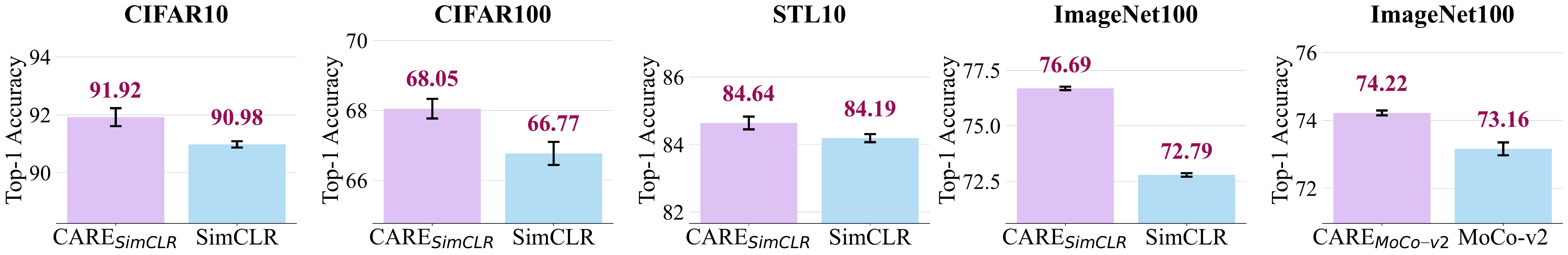}
    \caption{Top-1 linear readout accuracy (\%) on CIFAR10, CIFAR100, STL10 and ImageNet100. All results are from 5 independent seed runs for the linear probe. }
    \label{fig:accuracy_bar_plot}
    \vspace{-10pt}
\end{figure}

\renewcommand*\arraystretch{1.3} %

\subsection{Ablation of loss terms}

The \care loss $\mathcal{L}_{\text{\care}}$ is a weighted sum of the InfoNCE loss $\mathcal L_{\text{InfoNCE}}$ and the orthogonal equivariance loss $\mathcal L_{\text{equi}}$. Furthermore, as outlined in Section \ref{sec: main method}, the InfoNCE loss is itself a combination of an invariance inducing loss $\mathcal L_\text{inv}$ and a non-collapse term $\mathcal L_\text{unif}$. 
To study each loss component, we pretrain ResNet-50 models on CIFAR10 using different combinations of the three losses. The results in Figure \ref{fig: losses ablation} suggest that simply optimizing for $\mathcal L_\text{inv}$ and $\mathcal L_{\text{equi}}$ leads to collapse, while optimizing $\mathcal L_\text{unif}$ alone prevents collapse but performs similar to random initialization. Interestingly, $\mathcal L_\text{unif} + \mathcal L_{\text{equi}}$ yields non-trivial representations without directly enforcing invariance. But the performance falls below that of invariance-based contrastive baselines. In combination with the invariance term $\mathcal L_\text{inv}$---which biases rotations to be small---we achieve superior performance to the invariance-only counterpart.

\section{Related work}\label{sec: related work}

\textbf{Geometry of representations.} Equivariance is a key tool for encoding geometric structure---e.g., symmetries---into neural network representations  \citep{cohen2016group,bronstein2021geometric}.  Whilst hard-coding equivariance into model architectures is very successful, approximate learned equivariance  \citep{kaba2022equivariance,shakerinavastructuring}, has certain advantages: 1) when the symmetry is provided only by data, with no closed-form expression, 2) can still be used when it is unclear how to hard code equivariance into the architecture, and 3) can exploit standard high capacity architectures \citep{he2016deep,dosovitskiy2020image}, benefiting from considerable engineering efforts to optimize their performance. \cite{shakerinavastructuring} also consider learning orthogonal equivariance, but consider problems where both input and embedding space are acted on by $O(d)$. Our setting differs from this in two key ways: 1) we consider a very different set of transforms of input space---jitter, crops, etc.---and 2) can be naturally integrated into contrastive learning, and 3) theoretically study the minima of the angle-preserving loss.  A related line of work, \emph{mechanistic interpretability}, hypothesizes that algorithmic structure---possibly including group symmetries---emerge naturally within network connections during training \citep{chughtai2023toy}. Our approach is very different from this as we directly \emph{train} models to have the desired structure without relying on implicit processes. Finally, the geometry of representation space has been used in a very different sense in prior contrastive learning approaches, for instance bootstrapping useful negatives \cite{chuang2020debiased,robinson2020contrastive} based on their location in embedding space during training.

\noindent \textbf{Self-superised learning.} %
Prior equivariant contrastive learning approaches extend the usual setup of learning invariance by learning \emph{sensitivity} to certain features known to be important for downstream tasks. 
For instance, \cite{dangovski2021equivariant} learns to predict the augmentation applied but only considers a discrete group of 4-fold rotations. \cite{lee2021improving} learns the difference of augmentation parameters and \cite{xiao2020should} constructs separate embedding sub-spaces that capture invariances to all but one augmentation. However, these approaches do not offer a meaningful structure to the embedding space. Others attempt to control how this sensitivity occurs. Specifically, \cite{devillers2022equimod,garrido2023self,bhardwaj2023steerable} learn a mapping from one latent representation to another, predicting how data augmentation affects the embedding. 
But this does not constrain the group action on embeddings, resulting in complex non-linear augmentation maps. Finally, the recent work \cite{suau2023duet} implements approximate equivariance using 2D representations.

\section{Discussion}

Converting transformations that are complex in input space into simple transformations in embedding space has many potential uses.
For instance, modifying data (e.g., in order to reason about counterfactuals) can be viewed as transforming one embedding to another. If the sought after transformation was \emph{simple} and \emph{predictable}, it may be easier to find. Similarly, generalizing out-of-distribution is easier when extrapolating linearly \citep{xu2020neural}, suggesting that linear transformations of embedding space may facilitate more reliable generalization. This work considers several design principles that may be broadly relevant: 1) \emph{learned} equivariance preserves the expressivity of backbone architectures, and in some cases may be easier for model design than hard-coded equivariance, 2) linear group actions are desirable, but require carefully designed objectives (similar in spirit to the principle of \emph{parsimony} \citep{ma2022principles}, also advocated for by \cite{shakerinavastructuring}), and 3) orthogonal (and related) symmetries are a promising structure for Siamese network training as they can be efficiently learned using \emph{pair-wise} data comparisons.

\section{Acknowledgements}
\par This research was supported by NSF award CCF-2112665. Sharut Gupta is supported by MIT Presidential Fellowship. Derek Lim is supported by National Science Foundation Graduate Research Fellowship. 
\par We acknowledge MIT SuperCloud and Lincoln Laboratory Supercomputing Center
(Reuther et al., 2018) for providing HPC resources that have contributed to this work. We wish to thank Michael Murphy for insightful discussions on extensions of our method to biology.

\section{Reproducibility statement}
Algorithm \ref{alg:main_algorithm} in Appendix \ref{sec:algo_pseudocode} provides the pseudocode for implementing our work using the PyTorch framework. To ensure reproducibility, Appendix \ref{app: precise_exp_deets} details all the experimental configurations employed in our work. Additionally, our code is available at \href{https://github.com/Sharut/CARE}{\texttt{https://github.com/Sharut/CARE}}.

\bibliographystyle{abbrvnat}
\bibliography{bib}

\newpage
\appendix

\section{Proofs of theoretical results}\label{appendix: proofs}

The aim of this section is to detail the proofs of the theoretical results presented in the main manuscript. The key theoretical tools driving our analysis are prepared separately in Section \ref{app: orthogonal theory}.

Throughout our analysis, we assume that all spaces (e.g., $\mathcal A$ and $\mathcal X$) are subspaces of Euclidean space and therefore admit a Lebesgue measure. We also assume that all distributions (e.g., $a \sim \mathcal A$ and $x \sim \mathcal X$) admit a density with respect to the Lebesgue measure. With these conditions in mind, we recall the loss function that is the main object of study:
\begin{equation}\mathcal L_\text{equi}(f) = \mathbb{E}_{a \sim \mathcal A} \mathbb{E}_{x,x' \sim \mathcal X} \big  [ f(a(x'))^\top f(a(x)) - f(x) ^\top f(x') \big ]^2\end{equation}
Next, we re-state and prove Proposition \ref{prop: equivariance loss learns rotations}, our first key result.

\rotationthm*
\begin{proof}
    Suppose that $ \mathcal L_\text{equi}(f) =0$. This means that $f(a(x'))^\top f(a(x)) = f(x) ^\top f(x')$ for almost all $a \in G$, and $x, x' \in \mathcal X$. Setting $g_a(x) = f(a(x))$, we have that $g_a(x')^\top g_a(x) = f(x) ^\top f(x')$. The continuous version of the First Fundamental Theorem of invariant theory for the orthogonal group  (see Proposition \ref{thm: first fundamental theorem continuous version}) implies that there is an $ R_{a} \in O(d)$ such that $f(a(x))=g_a(x)=R_{a} f(x)$.
\end{proof}
As discussed in greater detail in the main manuscript, these results show that minimizing $\mathcal L_\text{equi}$ produces a model where an augmentation $a$ corresponds to a single orthogonal transformation of embeddings $R_a$, independent of the input. This result is continuous in flavor as it studies the loss over the full data distribution $p(x)$. There exists a corresponding result for the finite sample loss
\begin{equation*}
   \mathcal L_{\text{equi},n}(f) = \mathbb{E}_{a \sim \mathcal A} \sum_{i,j=1}^n \big  [ f(a(x_j))^\top f(a(x_i)) - f(x_i) ^\top f(x_j) \big ]^2.
\end{equation*}
\begin{prop}\label{prop: equivariance loss learns rotations finite sample version}
    Suppose $ \mathcal L_{\text{equi},n}(f) =0$. Then for almost every $a \in \mathcal{A}$, there is an orthogonal matrix $R_{a} \in O(d)$ such that $f(a(x_i)) = R_{a}f(x_i)$ for all $i =1, \ldots , n$.
\end{prop}
As for the population counterpart, the proof of this result directly follows from the application of the First Fundamental Theorem of invariant theory for the orthogonal group.
\begin{proof}[Proof of Proposition \ref{prop: equivariance loss learns rotations finite sample version}]
    Suppose that $ \mathcal L_\text{equi}(f) =0$. This means that for almost every $a \in G$, and every $i,j = 1 , \ldots , n$ we have $f(a(x_j))^\top f(a(x_i)) = f(x_i) ^\top f(x_j)$. In other words $AA^T = BB^T$ where $A,B \in \mathbb{R}^{n \times d}$ are matrices whose $i$th rows are $A_i = f(a(x_i))^\top$  and  $B_i = f(x_i)^\top$  respectively. This implies, by the First Fundamental Theorem of invariant theory for the orthogonal group (see Corollary \ref{cor: first fundamental theorem continuous version}), that there is an $ R_{a} \in O(d)$ such that $A=B R_{a}$. Considering only the $i$th rows of $A$ and $B$ leads us to conclude that $f(a(x_i)) = R_{a}f(x_i)$. 
\end{proof}
A corollary of Proposition \ref{prop: equivariance loss learns rotations} is that compositions of augmentations correspond to compositions of rotations.
\corrgroup*
\begin{proof}
    Applying Proposition \ref{prop: equivariance loss learns rotations} on $a' \circ a$ as the sampled augmentation,   we have that $f(a' \circ a(x_i)) = R_{a' \circ a}f(x_i) = \rho(a' \circ a) f(x_i)$. However, taking $\bar{x} = a(x_i)$ and applying Proposition \ref{prop: equivariance loss learns rotations} twice we also know that $f(a' \circ a(x_i))= f(a' (\bar{x})) = R_{a} f(\bar{x}) =  R_{a'}f( a(x_i)) =  R_{a'} R_{a} f(x) = \rho (a') \rho(a) f(x_i)$. That is, $\rho(a' \circ a) f(x_i) = f(a' \circ a(x_i))=  \rho (a') \rho(a) f(x_i)$. Since this holds for all $i$, we have that   $\rho(a' \circ a) = \rho (a') \rho(a)$. 
\end{proof}

\par This corollary requires us to assume that $\mathcal{A}$ is a semi-group. That is, $\mathcal{A}$  is closed under compositions, but group elements do not necessarily have inverses and it does not need to include an identity element. 

\section{Background on invariance theory for the orthogonal group}\label{app: orthogonal theory}

This section recalls some classical theory on orthogonal groups and an extension that we use for proving results over continuous data distributions. 

A function $f : (\mathbb{R}^d)^n \rightarrow \mathbb{R}$ is said to be $O(d)$-invariant if $f(Rv_1, \ldots , Rv_n) = f(v_1, \ldots , v_n)$ for all $R \in O(d)$. Throughout this section, we are especially interested in determining easily computed statistics that \emph{characterize} an $O(d)$ invariant function $f$. In other words, we would like to write $f$ as a function of these statistics. The following theorem was first proved by Hermann Weyl using Capelli's identity \citep{weyl1946classical} and shows that the inner products $v_i^\top v_j$ suffice. 

\begin{thm}[First fundamental theorem of invariant theory for the orthogonal group]\label{thm: 1st fundamental thm of O(d)}
    Suppose that $f : (\mathbb{R}^d)^n \rightarrow \mathbb{R}$ is $O(d)$-invariant. Then there exists a function $g : \mathbb{R}^{n \times n} \rightarrow \mathbb{R}$ for which
    \begin{equation*}
        f(v_1, \ldots , v_n ) = g \big ( [v_i^\top v_j ]_{i,j=1}^n\big ).
    \end{equation*}
\end{thm}

In other words, to compute $f$ at a given input, it is not necessary to know all of $v_1, \ldots, v_n$. Computing the value of $f$ at a point can be done using only  the inner products $v_i^\top v_j$, which are invariant to $O(d)$. Letting $V$ be the $n \times d$ matrix whose $i$th row is $v_i^\top$, we may also write $f(v_1, \ldots , n_n ) = g(VV^\top)$. The map $V \mapsto VV^\top$ is known as the orthogonal projection of $V$. 

A corollary of this result has recently been used to develop $O(d)$ equivariant architectures in machine learning \citep{villar2021scalars}.
\begin{cor}\label{cor: first fundamental theorem continuous version}
    Suppose that $A,B$ are $n \times d$ matrices and $AA^\top = BB^\top$. Then $A=BR$ for some $R \in O(d)$.
\end{cor}
 \cite{villar2021scalars} use this characterization of orthogonally equivariant functions to \emph{parameterize} function classes of neural networks that have the same equivariance. This result is also useful in our context; However, we put it to use for a very different purpose: studying $\mathcal L_\text{equi}$.

Intuitively this result says the following: given two point clouds $A,B$ of unit length vectors with some fixed correspondence (bijection) between each point in $A$ and a point in $B$, if the \emph{angles} between the $i$th and $j$th points in cloud $A$ always equal the angle between the $i$th and $j$th point in cloud $B$, then $A$ and $B$ are the same up to an orthogonal transformation.

This is the main tool we use to prove the finite sample version of the main result for our equivariant loss (Proposition \ref{prop: equivariance loss learns rotations finite sample version}). However, to analyze the population sample loss $\mathcal L_\text{equi}$ (Proposition \ref{prop: equivariance loss learns rotations}), we require an extended version of this result to the continuous limit as $n \rightarrow \infty$. To this end, we develop a simple but novel extension to Theorem \ref{thm: 1st fundamental thm of O(d)} to the case of continuous data distributions. This result may be useful in other contexts independent of our setting.  

\begin{prop}\label{thm: first fundamental theorem continuous version}
    Let $\mathcal X$ be any set and $f,h: \mathcal X \rightarrow \mathbb{R}^d$ be functions on $\mathcal X$. If $f(x)^\top f(y) = h(x)^\top h(y)$ for all $x,y \in \mathcal X$, then there exists $R \in O(d)$ such that $Rf(x)=h(x)$ for all $x \in \mathcal X$.
\end{prop}
The proof of this result directly builds on the finite sample version. The key idea of the proof is that since the embedding space $\mathbb{R}^d$ is finite-dimensional we may select a set of points $\{f(x_i)\}_i$ whose span has maximal rank in the linear space spanned by the outputs of $f$. This means that any arbitrary point $f(x)$ can be written as a linear combination of the $f(x_i)$. This observation allows us to apply the finite sample result on each $f(x_i)$ term in the sum to conclude that $f(x)$ is also a rotation of a sum of $h(x_i)$ terms. Next, we give the formal proof.
\begin{proof}[Proof of Proposition \ref{thm: first fundamental theorem continuous version}]
    Choose $x_1, \ldots , x_n \in \mathcal X$ such that $F = [f(x_1) \mid \ldots \mid f(x_n)]^\top \in \mathbb{R}^{n \times d}$ and  $h = [h(x_1) \mid \ldots \mid h(x_n)]^\top \in \mathbb{R}^{n \times d}$ have maximal  rank. Note we use ``$\mid$'' to denote the column-wise concatenation of vectors. Note that such $x_i$ can always be chosen. Since we have $FF^\top =  HH^\top$, we know by Corollary \ref{cor: first fundamental theorem continuous version} that $F=HR$  for some  $R \in O(d)$. 
    
    Now consider an arbitrary $x \in \mathcal X$ and define $\tilde{F} = [F \mid f(x)]^\top $ and $\tilde{H} = [H \mid h(x)]^\top $, both of which belong to $\mathbb{R}^{(n+1) \times d}$. Note that again we have $\tilde{F}\tilde{F}^\top =  \tilde{H}\tilde{H}^\top$ so also know that $\tilde{F}=\tilde{H}\tilde{R}$  for some  $\tilde{R} \in O(d)$.  Since $x_i$ were chosen so that $F$ and $H$ are of maximal rank, we know that $h(x) = \sum_{i=1}^n c_i  h(x_i)$ for some  coefficients $c_i \in \mathbb{R}$, since if this were not the case then we would have $\text{rank}(\tilde{H}) = \text{rank}(H)+1$.

    From this, we know that
    
    \begin{align*}
        R^\top h(x)  &=  \sum_{i=1}^n c_i  R^\top h(x_i) \\
                    &=   \sum_{i=1}^n c_i f(x_i) \\
                    &=  \sum_{i=1}^n c_i  \tilde{R}^\top h(x_i) \\
                    &= \tilde{R}^\top \sum_{i=1}^n c_i  h(x_i) \\
                      &= \tilde{R}^\top h(x) \\
                      &= f(x).
    \end{align*}

    So we have that $Rf(x) = RR^\top h(x) = h(x)$  for all $x \in \mathcal X$.
\end{proof}

\section{Extensions to other groups: further discussion} \label{app: extension to other groups}
In Section \ref{sec: extension to other groups}, we explore the possibility of formulating an equivariant loss $\mathcal L_{\text{equi}}$ for pairs of points that fully captures equivariance by requiring the group to be the stabilizer of a bilinear form. In this context, the invariants are generated by polynomials of degree two in two variables, and the equivariant functions can be obtained by computing gradients of these invariants \citep{blum2022equivariant}. Section \ref{sec: extension to other groups} notes that this holds true not only for the orthogonal group, which is the primary focus of our research but also for the Lorentz group and the symplectic group, suggesting natural extensions of our approach. 

\par It is worth noting that the group of rotations $SO(d)$ does not fall into this framework. It can be defined as the set of transformations that preserve both inner products (a 2-form) and determinants (a $d$-form). Consequently, some of its generators have degree 2 while others have degree $d$ (see \citep{weyl1946classical}, Section II.A.9).
\par Weyl's theorem states that if a group acts on $n$ copies of a vector space (in our case, $(\mathbb R^d)^n$ for consistency with the rest of the paper), its action can be characterized by examining how it acts on $k$ copies (i.e., $(\mathbb R^d)^k$) when the maximum degree of its irreducible components is $k$ (refer to Section 6 of \citep{schmid2006finite} for a precise statement of the theorem). Since our interest lies in understanding equivariance in terms of pairs of objects, we desire invariants that act on pairs of points. One way to guarantee this is to restrict ourselves to groups that act through representations where the irreducible components have degrees of at most two (though this is not necessary in all cases, such as the orthogonal group $O(d)$ that we consider in the main paper). An example of such groups is the product of finite subgroups of the unitary group $U(2)$, which holds relevance in particle physics. According to Weyl's theorem, the corresponding invariants can be expressed as \emph{polarizations} of degree-2 polynomials on two variables. Polarizations represent an algebraic construction that enables the expression of homogeneous polynomials in multiple variables by introducing additional variables to polynomials with fewer variables. In our case, the base polynomials consist of degree-2 polynomials in two variables, while the polarizations incorporate additional variables. Notably, an interesting open problem lies in leveraging this formulation for contrastive learning.

\section{Implementation details}\label{sec:algo_pseudocode}
Algorithm \ref{alg:main_algorithm} presents pytorch-based pseudocode for implementing \care. This implementation introduces the idea of using a smaller batch size for the equivariance loss compared to the InfoNCE loss. Specifically, by definition, the equivariance loss is defined as a double expectation, one over data pairs and the other over augmentations. Empirical observations reveal that sampling one augmentation per batch leads to unstable yet superior performance when compared to standard invariant-based baselines such as SimCLR. Since these invariant-based contrastive benchmarks generally perform well with large batch sizes, we adopt the approach of splitting a batch into multiple chunks to efficiently sample multiple augmentations per batch for the equivariance loss. Each chunk of the batch is associated with a new pair of augmentations, ensuring a large batch size for the InfoNCE loss and a smaller batch size for the equivariance loss.

\begin{algorithm}[htb!]
\small
\caption[]{PyTorch based pseudocode for \care}
\label{alg:main_algorithm}
\begin{algorithmic}[1]
\State \textbf{Notations: } $f$ represents the backbone encoder network, $\lambda$ is the weight on \care loss, $\texttt{apply\_same\_aug}$ function applies the same augmentation to all samples in the input batch

\For {minibatch $x$ in $\texttt{dataloader}$}
    \State draw \emph{two batches} of augmentation functions ${a}_{1}, {a}_{2} \in \mathcal{A}$
    \State \Comment{Functions ${a}_{1}, {a}_{2}$ apply different augmentation to each sample in batch $x$}
    \State $z^{\text{inv}}_1, z^{\text{inv}}_2 = f(a_1(x)), f(a_2(x))$
    \State divide $x$ into $\texttt{n\_split}$ chunks to form $x_{\text{chunks}}$ 
    \State \Comment{Module for calculating orthogonal equivariance loss}
    \For {$c_i$ in $x_{\text{chunks}}$ in parallel }
        \State draw \emph{two} augmentation functions $\tilde{a}_{1}, \tilde{a}_{2} \in \mathcal{A}$
        \State \Comment{Functions $\tilde{a}_{1}, \tilde{a}_{2}$ apply same augmentation to each sample in batch $c_i$}
        \State $\tilde{z}_{i1}, \tilde{z}_{i2} = f(\texttt{apply\_same\_aug}(c_i, \tilde{a}_{1})), f(\texttt{apply\_same\_aug}(c_i, \tilde{a}_{2}))$ 
    \EndFor
    \State \Comment{Concatenate embedding vectors corresponding to all chunks}
    \State merge $\tilde{z}_{i1}, \tilde{z}_{i2}$ into $z^{\text{equiv}}_{1}, z^{\text{equiv}}_{2}$ respectively 
    \State \Comment{Loss computation}
    
    \State $\mathcal{L}_{\text{InfoNCE}}(f) = \texttt{infonce\_loss}(z^{\text{inv}}_1, z^{\text{inv}}_2)$ 
    \State $\mathcal{L}_{\text{equiv}}(f) = \texttt{orthogonal\_equivariance\_loss}(z^{\text{equiv}}_{1}, z^{\text{equiv}}_{2}, \texttt{n\_split})$
    \State $\mathcal{L}_{\text{\care}}(f) = \mathcal{L}_{\text{InfoNCE}}(f) + \lambda \cdot \mathcal{L}_{\text{equiv}}(f)$
    \State \Comment{Optimization step}
    \State $\mathcal{L}_{\text{\care}}(f)$.backward()
    \State  optimizer.step()

\EndFor
\end{algorithmic}
\end{algorithm}

\section{Supplementary experimental details and assets disclosure}\label{app: experimental details}

\subsection{Assets} We do not introduce new data in the course of this work. Instead, we use publicly available widely used image datasets for the purposes of benchmarking and comparison.

\subsection{Hardware and setup}
All experiments were performed on an HPC computing cluster using 4 NVIDIA Tesla V100 GPUs with 32GB accelerator RAM for a single training run. The CPUs used were Intel Xeon Gold 6248 processors with 40 cores and 384GB RAM. All experiments use the
PyTorch deep learning framework \citep{paszke2019pytorch}.

\subsection{Experimental protocols}\label{app: precise_exp_deets}
We first outline the training protocol adopted for training our proposed approach on a variety of datasets, namely CIFAR10, CIFAR100, STL10, and ImageNet100.

\par \noindent \textbf{CIFAR10, CIFAR100 and STL10}\quad All encoders have ResNet-50 backbones and are trained for 400 epochs with temperature $\tau=0.5$ for SimCLR and $\tau=0.1$ for MoCo-v2 \footnote{\url{https://github.com/facebookresearch/moco}}. The  encoded features have a dimension of 2048 and are further processed by a two-layer MLP projection head, producing an output dimension of 128. A batch size of 256 was used for all datasets. For CIFAR10 and CIFAR100, we employed the Adam optimizer with a learning rate of $1e^{-3}$ and weight decay of $1e^{-6}$. For STL10, we employed the SGD optimizer with a learning rate of $0.06$, utilizing cosine annealing and a weight decay of $5e^{-4}$, with 10 warmup steps. We use the same set of augmentations as in SimCLR \citep{simclr}. To train the encoder using $\mathcal{L}_{\text{\care-SimCLR}}$, we use the same hyper-parameters for InfoNCE loss. Additionally, we use 4, 8 and 16 batch splits for CIFAR100, STL10 and CIFAR10, respectively. This allows us to sample multiple augmentations per batch, effectively reducing the batch size of equivariance loss whilst retaining the same for InfoNCE loss. Furthermore, for the equivariant term, we find it optimal to use a weight of $\lambda=0.01, 0.001$, and $0.01$ for CIFAR10, CIFAR100, and STL10, respectively. \\

\noindent\textbf{ImageNet100}\quad We use ResNet-50 as the encoder architecture and pretrain the model for 200 epochs. A base learning rate of 0.8 is used in combination with cosine annealing scheduling and a batch size of 512. For MoCo-v2, we use 0.99 as the momentum and $\tau=0.2$ as the temperature. All remaining hyperparameters were maintained at their respective official defaults as in the official MoCo-v2 code. While training with $\mathcal{L}_{\text{\care-SimCLR}}$ and $\mathcal{L}_{\text{\care-MoCo}}$, we find it optimal to use splits of 4 and 8 and weight of $\lambda=0.005$ and $0.01$ respectively on the equivariant term. 

\noindent\textbf{Linear evaluation}\quad We train a linear classifier on frozen features for 100 epochs with a batch size of 512 for CIFAR10, CIFAR100, and STL10 datasets. To optimize the classifier, we employ the Adam optimizer with a learning rate of $1e^{-3}$ and a weight decay of $1e^{-6}$. In the case of ImageNet100, we train the linear classifier for 60 epochs using a batch size of 128. We initialize the learning rate to 30.0 and apply a step scheduler with an annealing rate of 0.1 at epochs 30, 40, and 50. The remaining hyper-parameters are retained from the official code.

\section{Additional experiments}
\textbf{Histogram for loss ablation.} To accompany Figure \ref{fig: losses ablation}, this section plots the cosine similarity between positive pairs. We provide two plots for each experiment: the first plots the \emph{histogram} of similarities of positive pairs drawn from the test set; the second plots the \emph{average} positive cosine similarity throughout training. The results are reported in Figures \ref{fig:plot1}, \ref{fig:plot2}, \ref{fig:plot3}, \ref{fig:plot4}, \ref{fig:plot5}, \ref{fig:plot6}. 

\begin{figure}[htbp]
    \centering
    \begin{minipage}{0.35\textwidth}
        \centering
        \includegraphics[width=\linewidth]{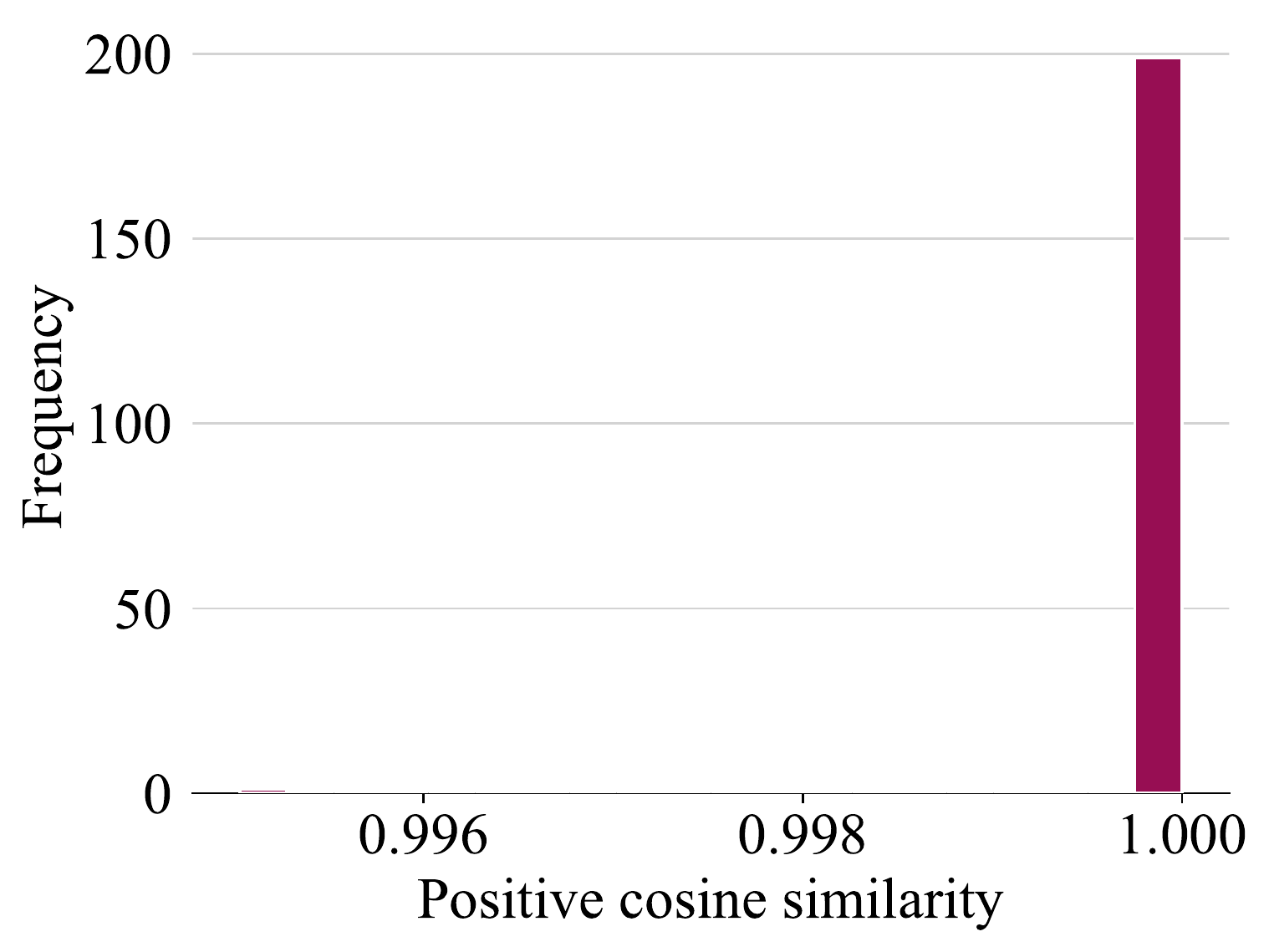}
    \end{minipage}%
    \begin{minipage}{0.35\textwidth}
        \centering    \includegraphics[width=\linewidth]{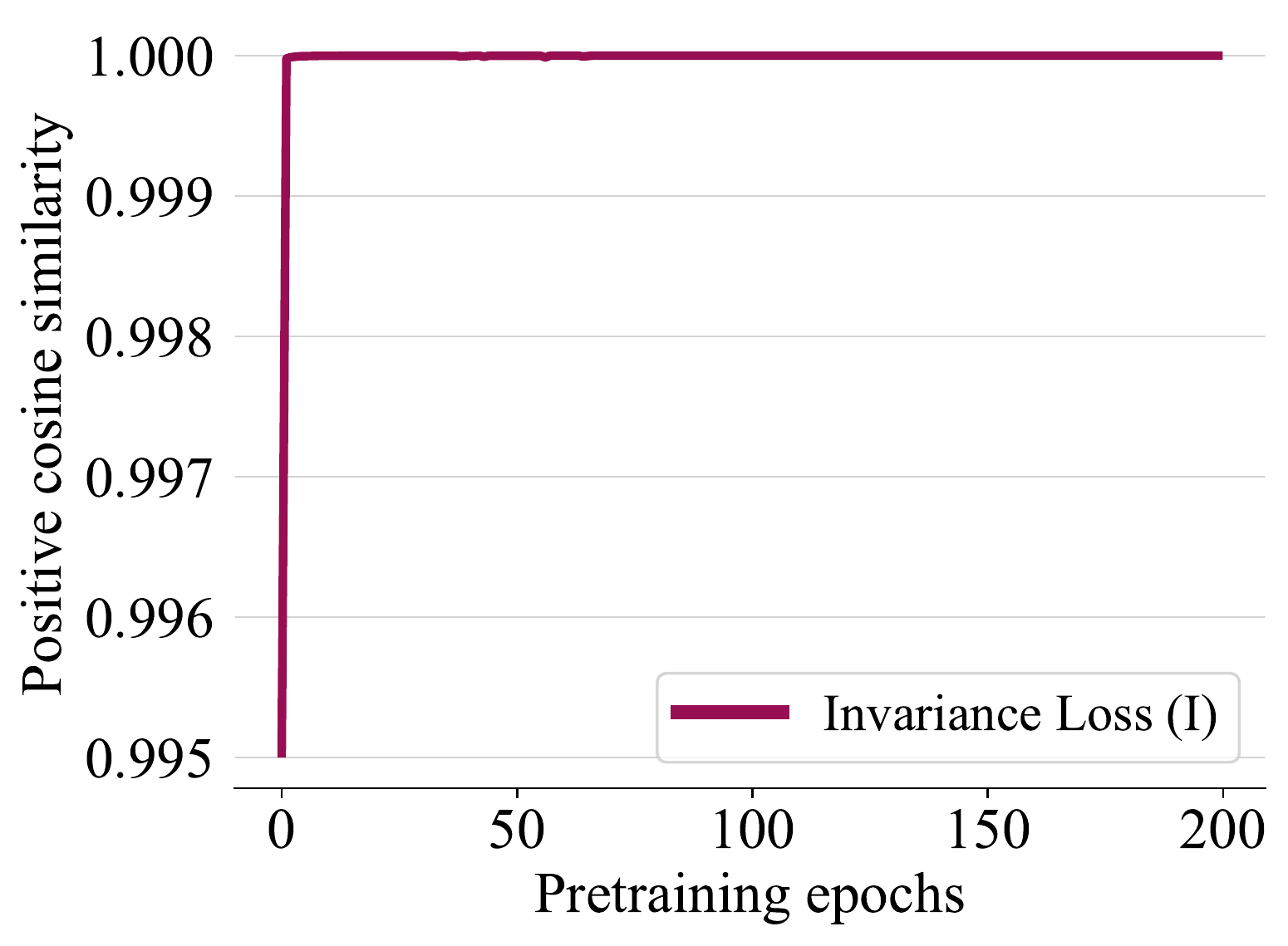}
    \end{minipage}
    \caption{(left) Histogram of positive cosine similarity values at the end of pre-training using the invariance loss; (right) Evolution of positive cosine similarity values over pre-training epochs using the invariance loss}
    \label{fig:plot1}
\end{figure}

\begin{figure}[htbp]
    \centering
    \begin{minipage}{0.35\textwidth}
        \centering      \includegraphics[width=\linewidth]{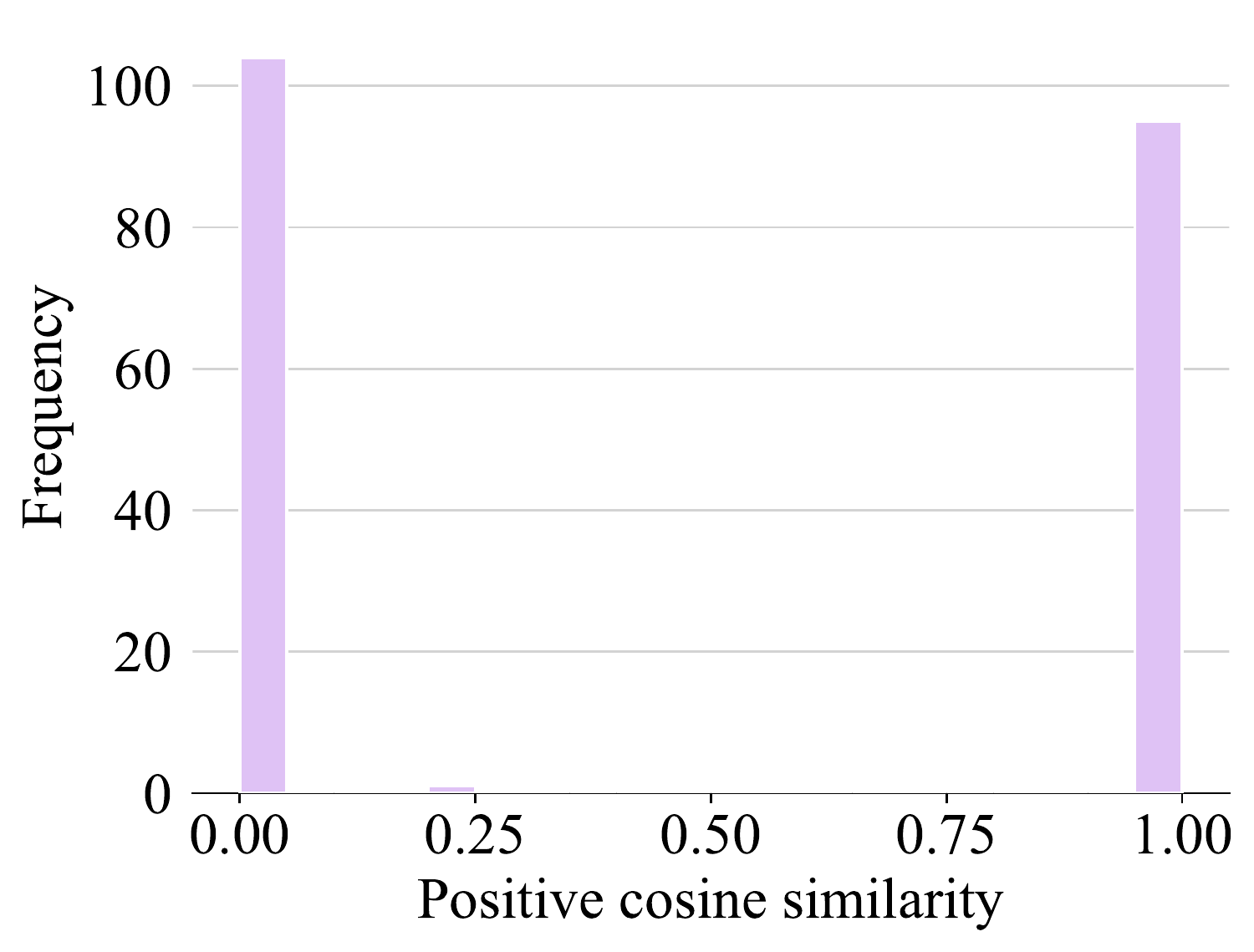}
    \end{minipage}%
    \begin{minipage}{0.35\textwidth}
        \centering
        \includegraphics[width=\linewidth]{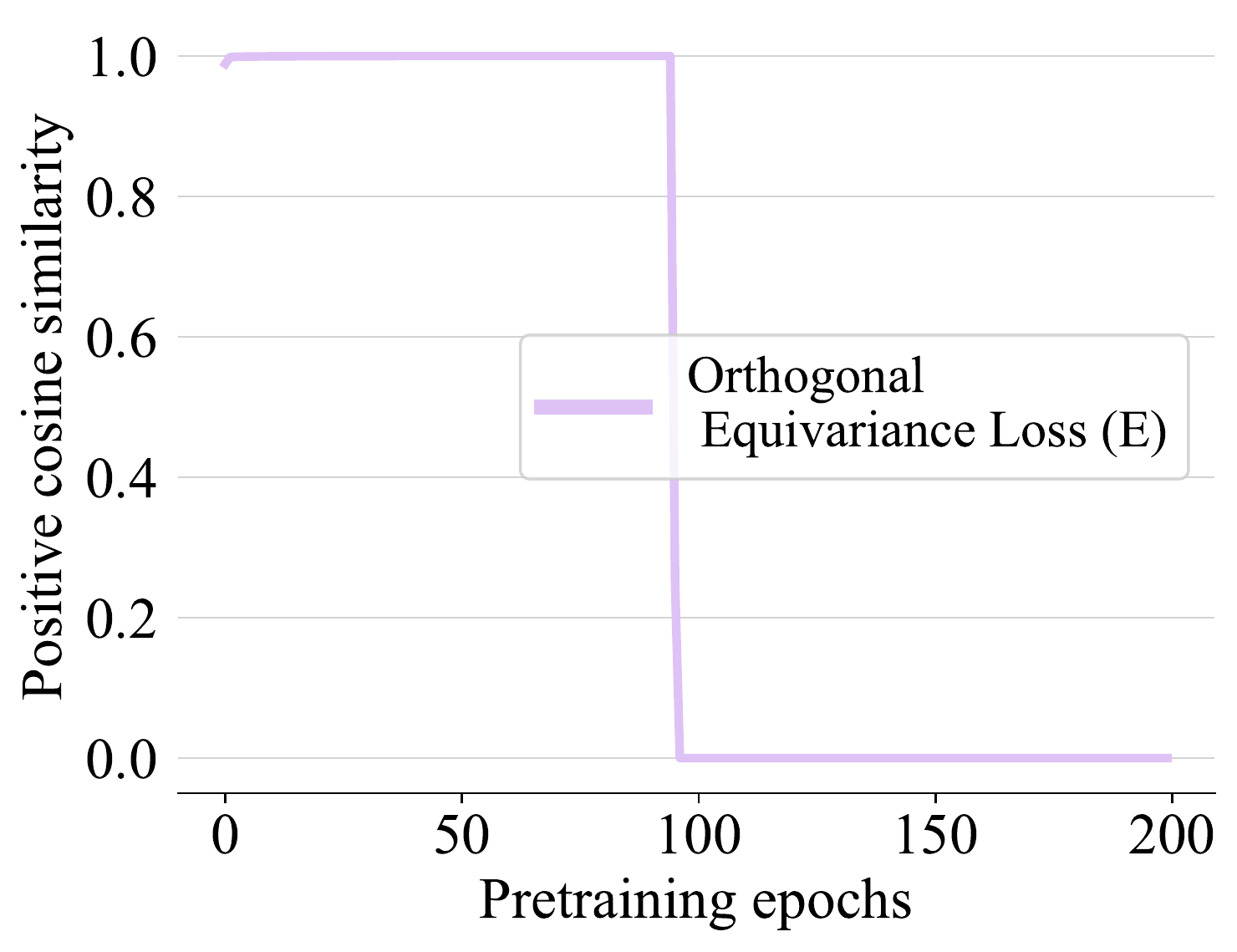}
    \end{minipage}
    \caption{(left) Histogram of positive cosine similarity values at the end of pre-training using the orthogonal equivariance loss; (right) Evolution of positive cosine similarity values over pre-training epochs using the orthogonal equivariance loss }
    \label{fig:plot2}
\end{figure}

\begin{figure}[htbp]
    \centering
    \begin{minipage}{0.35\textwidth}
        \centering        \includegraphics[width=\linewidth]{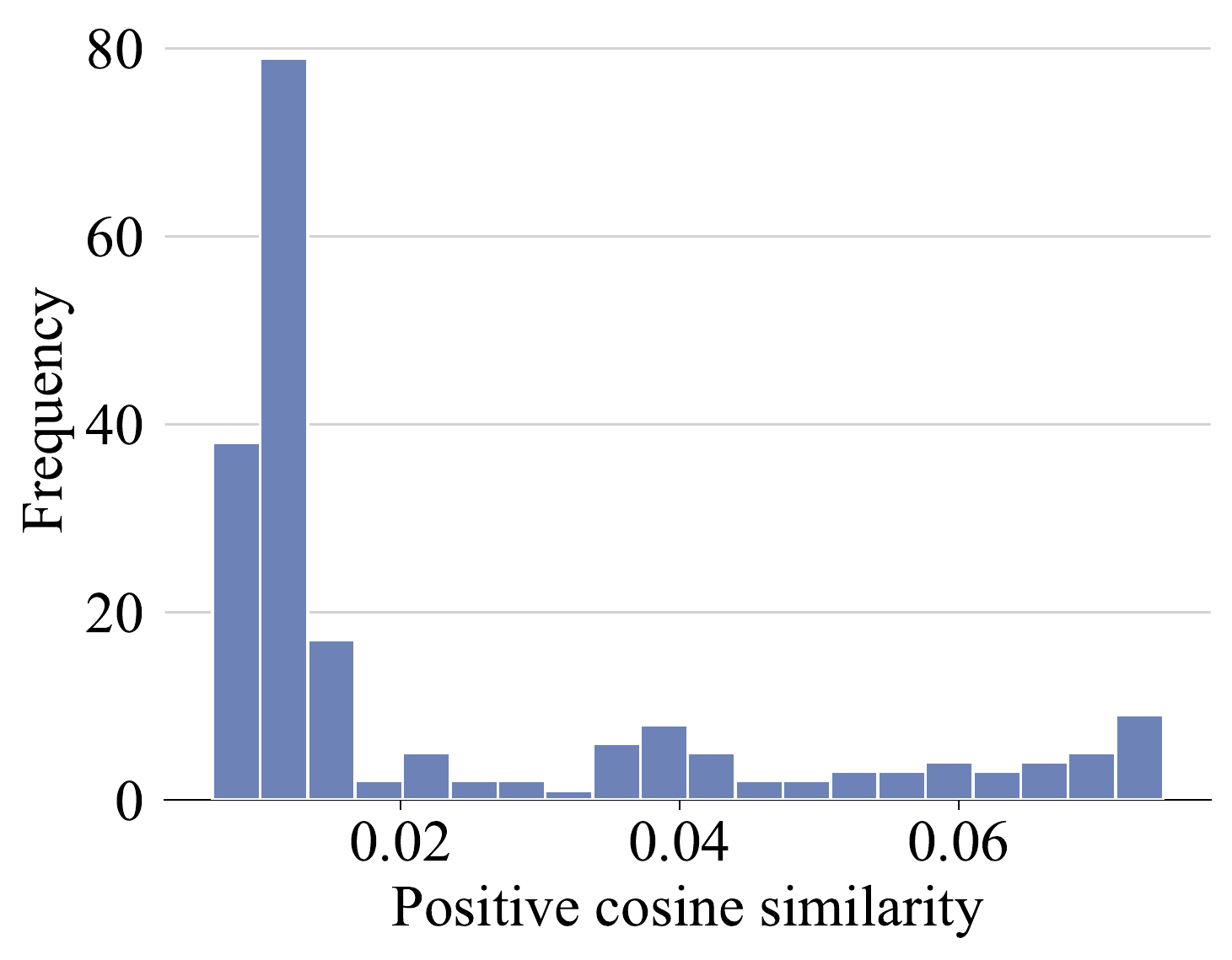}
    \end{minipage}%
    \begin{minipage}{0.35\textwidth}
        \centering        \includegraphics[width=\linewidth]{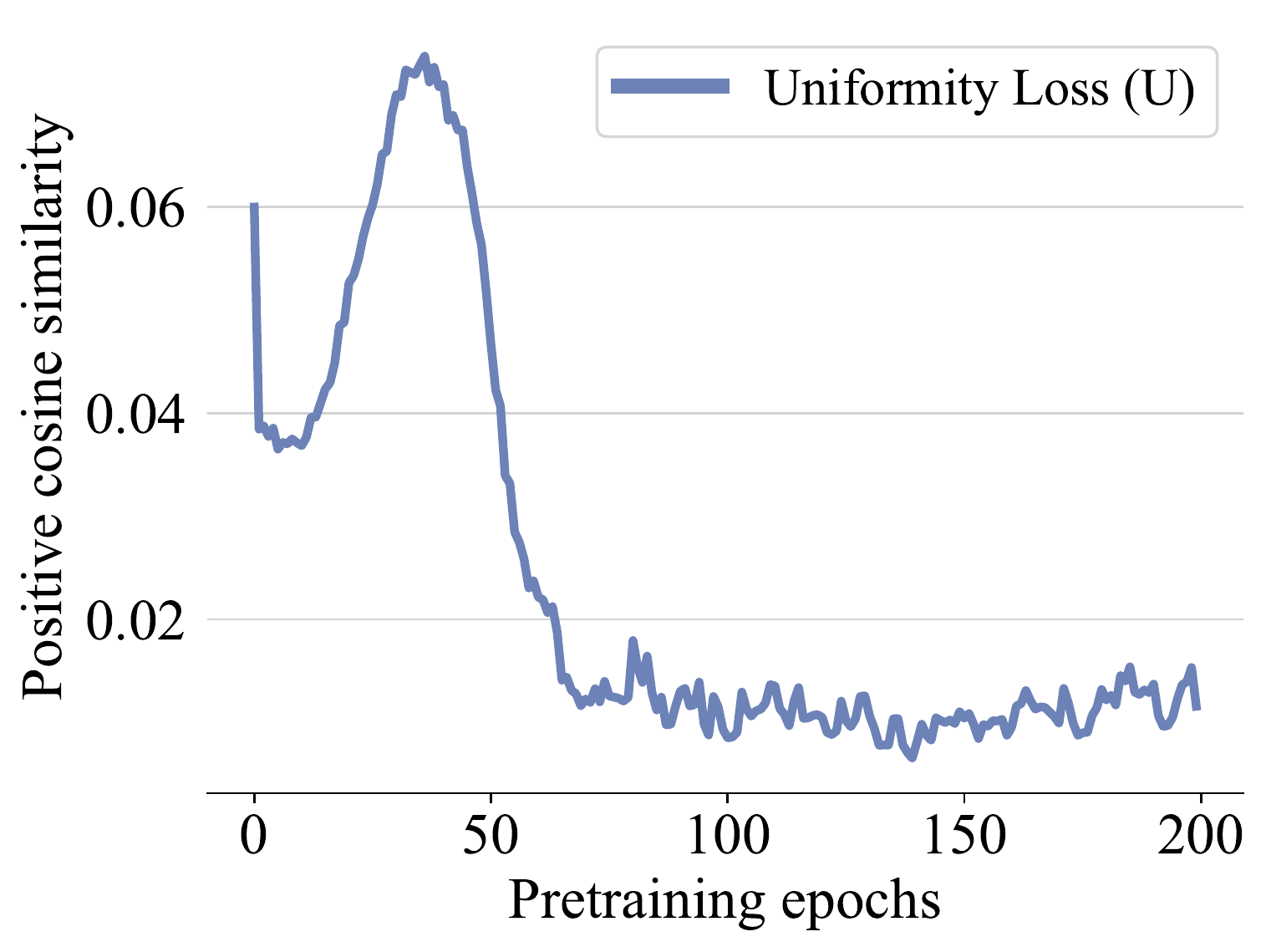}
    \end{minipage}
    \caption{(left) Histogram of positive cosine similarity values at the end of pre-training using the uniformity loss; (right) Evolution of positive cosine similarity values over pre-training epochs using the uniformity loss }
    \label{fig:plot3}
\end{figure}

\begin{figure}[htbp]
    \centering
    \begin{minipage}{0.35\textwidth}
        \centering        \includegraphics[width=\linewidth]{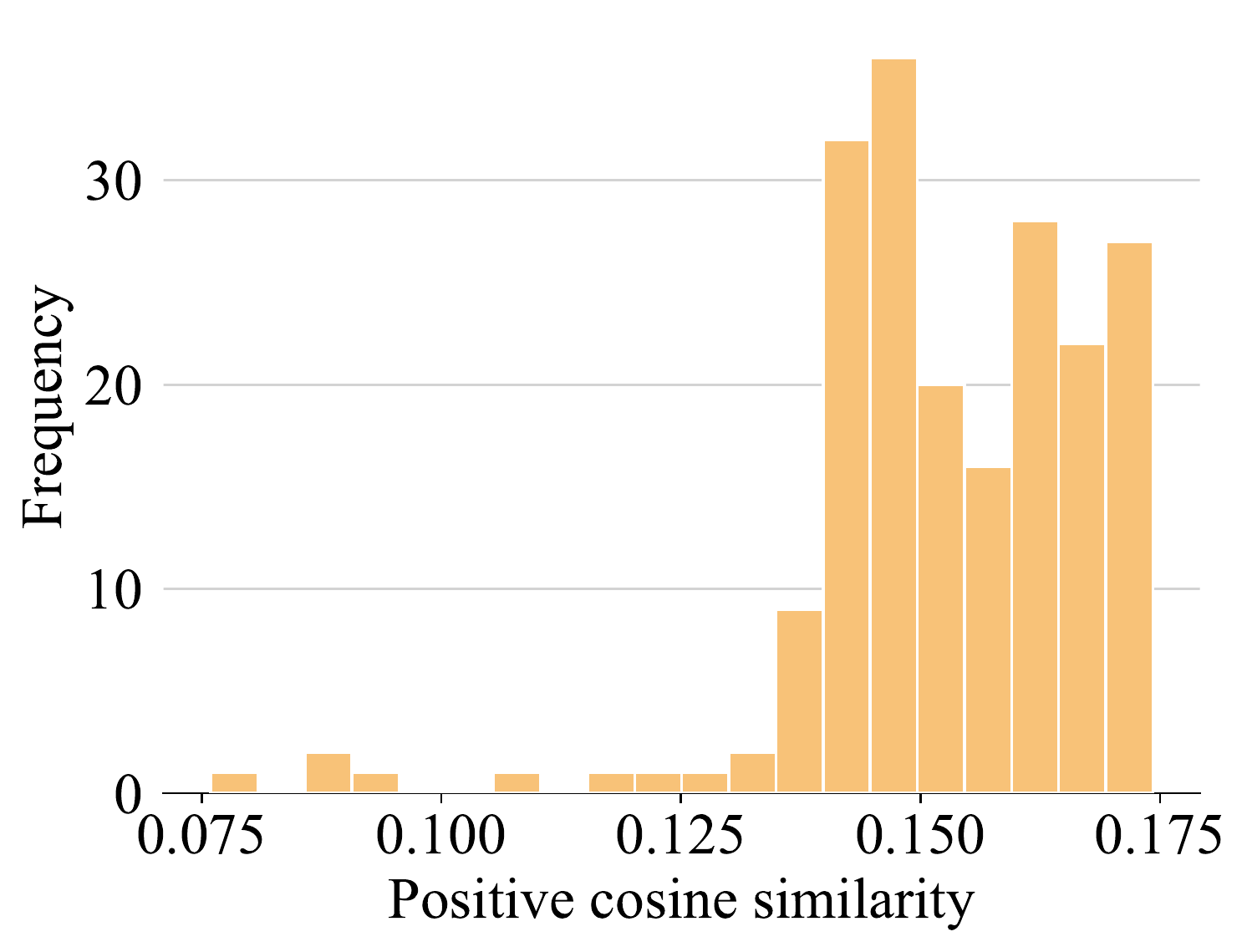}
    \end{minipage}%
    \begin{minipage}{0.35\textwidth}
        \centering        \includegraphics[width=\linewidth]{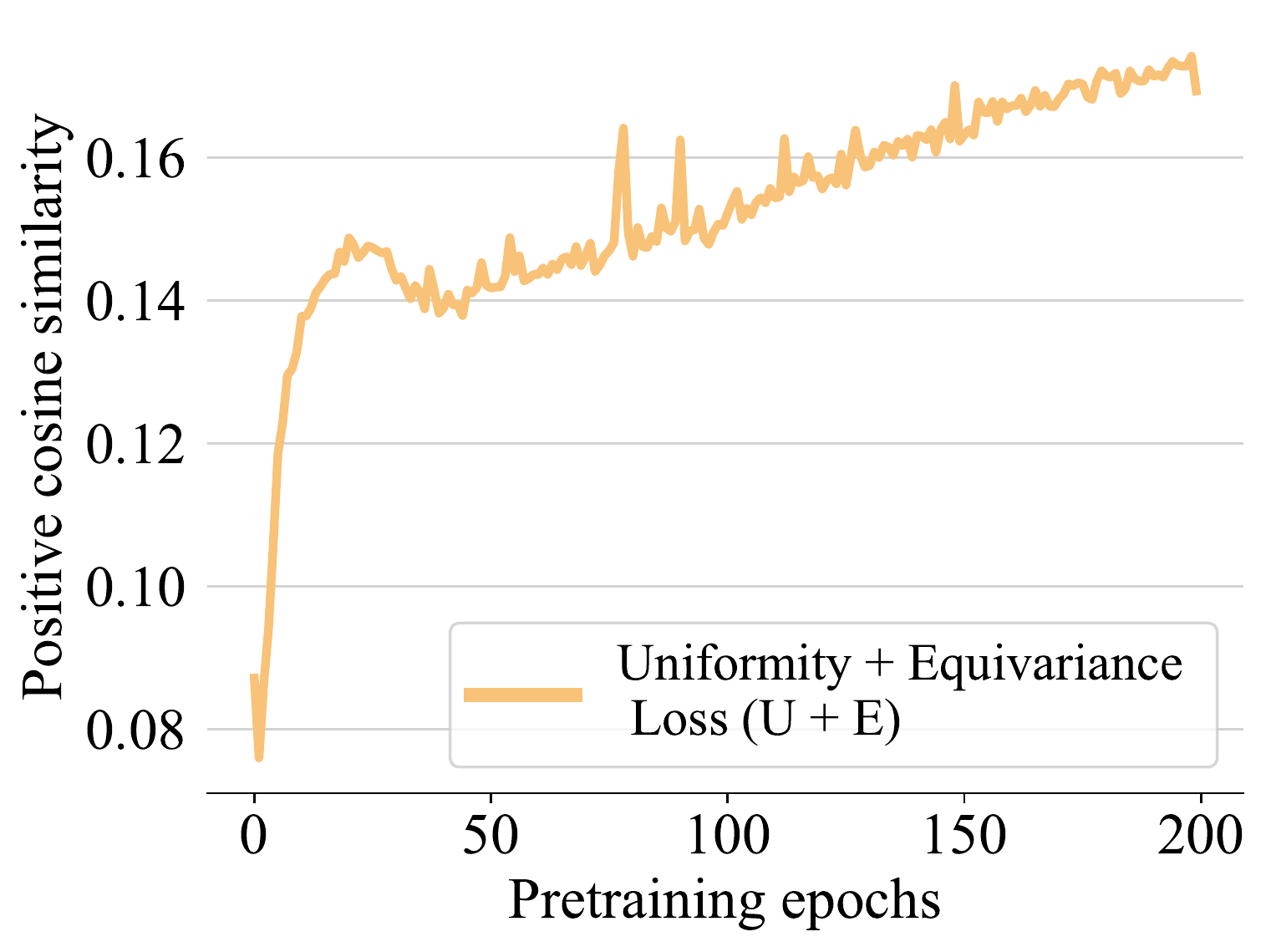}
    \end{minipage}
    \caption{(left) Histogram of positive cosine similarity values at the end of pre-training using the Uniformity + Equivariance loss; (right) Evolution of positive cosine similarity values over pre-training epochs using the Uniformity + Equivariance loss }
    \label{fig:plot4}
\end{figure}

\begin{figure}[htbp]
    \centering
    \begin{minipage}{0.35\textwidth}
        \centering        \includegraphics[width=\linewidth]{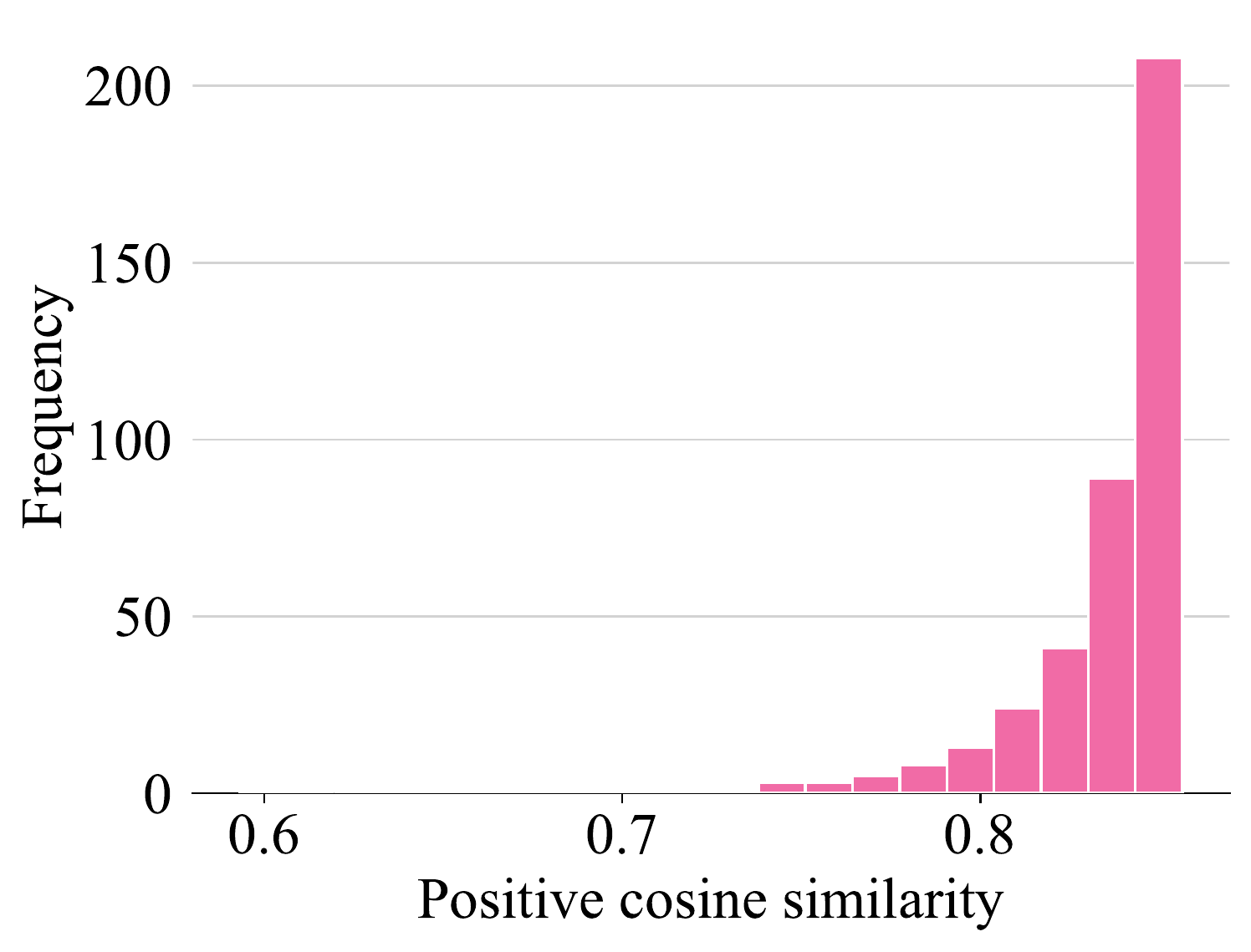}
    \end{minipage}%
    \begin{minipage}{0.35\textwidth}
        \centering       \includegraphics[width=\linewidth]{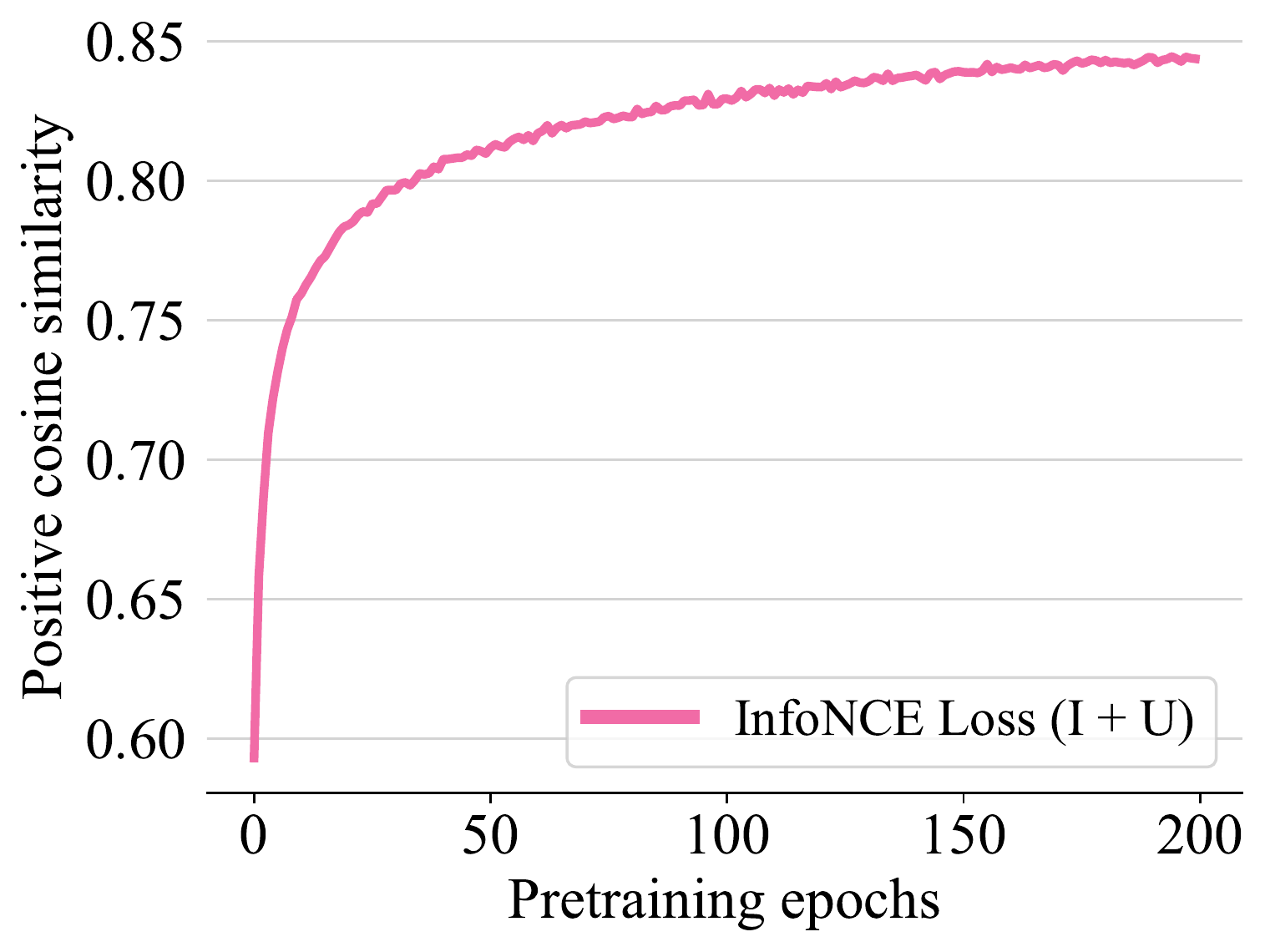}
    \end{minipage}
    \caption{(left) Histogram of positive cosine similarity values at the end of pre-training using the InfoNCE (invariance + uniformity) loss; (right) Evolution of positive cosine similarity values over pre-training epochs using the InfoNCE loss}
    \label{fig:plot5}
\end{figure}

\begin{figure}[htbp]
    \centering
    \begin{minipage}{0.35\textwidth}
        \centering
        \includegraphics[width=\linewidth]{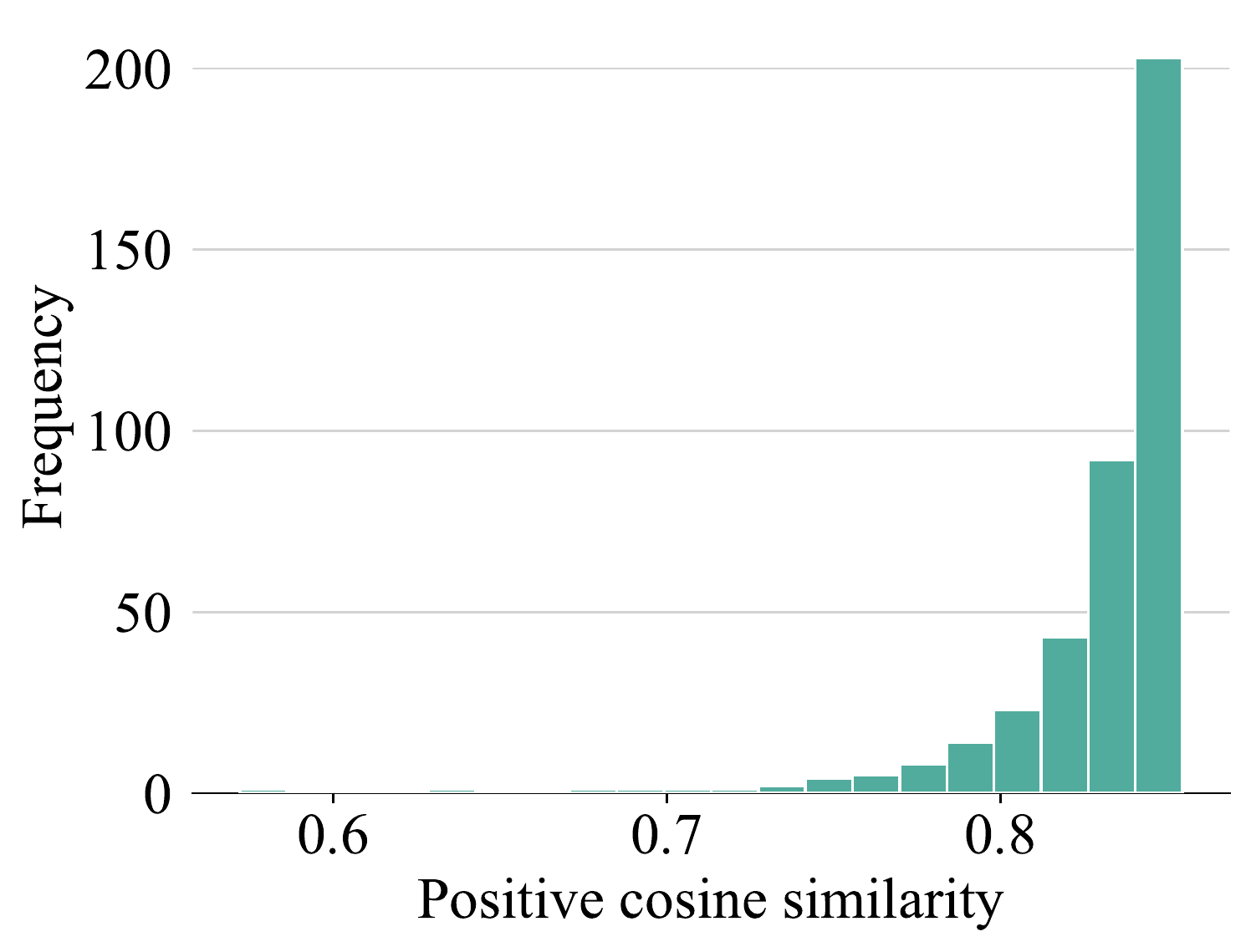}
    \end{minipage}%
    \begin{minipage}{0.35\textwidth}
        \centering       \includegraphics[width=\linewidth]{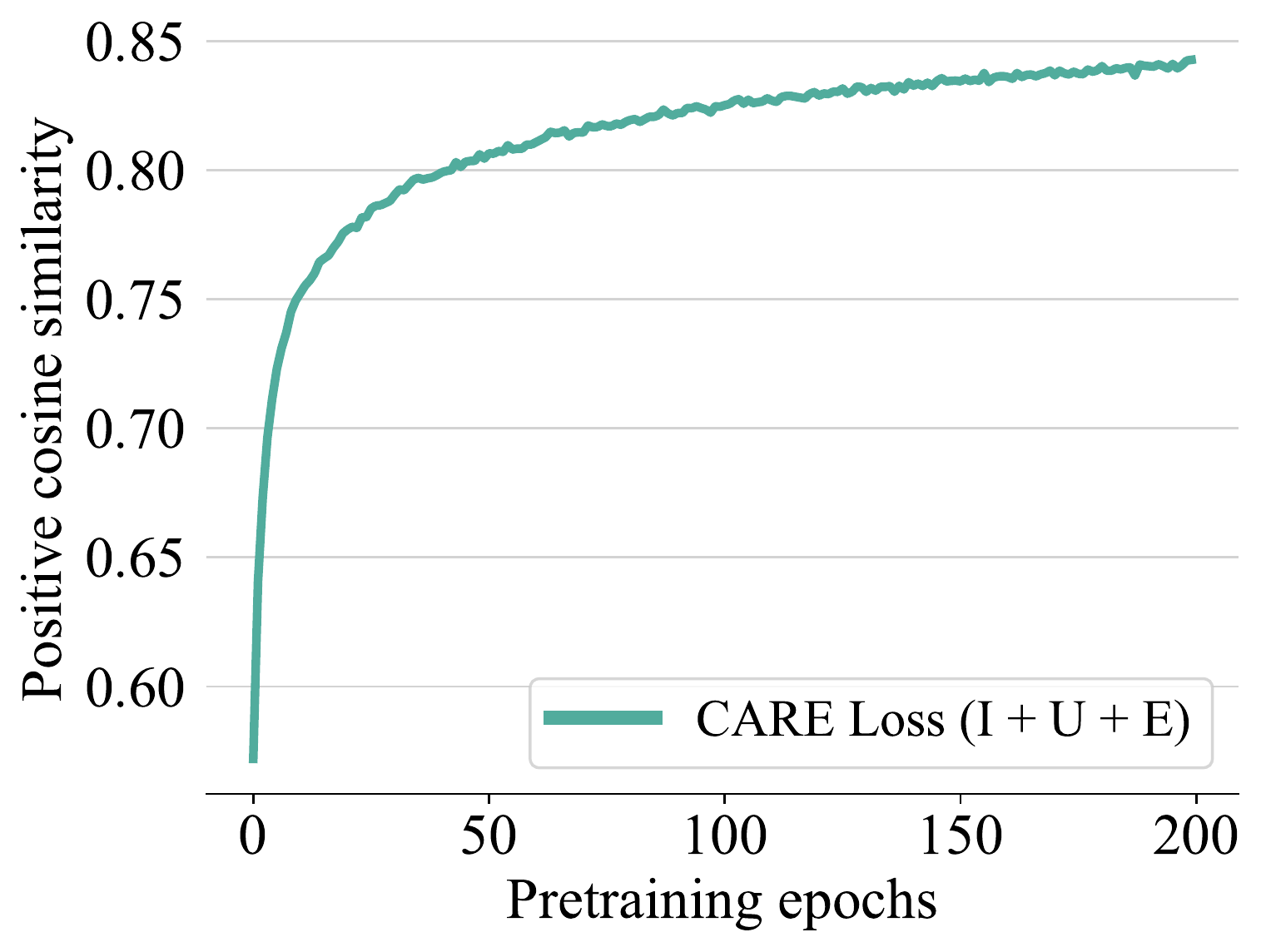}
    \end{minipage}
    \caption{(left) Histogram of positive cosine similarity values at the end of pre-training using the \care (InfoNCE + orthogonal equivariance) loss; (right) Evolution of positive cosine similarity values over pre-training epochs using the \care loss}
    \label{fig:plot6}
\end{figure}

\section{Additional discussion}\label{app: further discussion}

\textbf{Limitations.} While our method, \care, learns embedding spaces with many advantages over prior contrastive learning embedding spaces, there are certain limitations that we acknowledge here. First, we do not provide a means to directly identify the rotation corresponding to a specific transformation. Instead, our approach allows the recovery of the rotation by solving Wahba's problem. However, this requires solving an instance of Wahba's for each augmentation of interest.  Future improvements that develop techniques for quickly and easily (i.e., without needing to solve an optimization problem) identifying specific rotations would be a valuable improvement, enhancing the steerability of our models. Second, it is worth noting that equivariant contrastive methods, including \care, only achieve approximate equivariance. This is a fundamental challenge shared by all such methods, as it is unclear how to precisely encode exact equivariance. The question remains open as to a) whether this approximate equivariance should be considered damaging in the first place, and if so, b) whether scaling techniques can sufficiently produce reliable approximate equivariance to enable the diverse applications that equivariance promises. Addressing this challenge is a crucial area for future research and exploration in the field. Each of these limitations points to valuable directions for future work. 

\noindent\textbf{Broader impact.} 
Through our self-supervised learning method \care we explore foundational questions regarding the structure and nature of neural network representation spaces. Currently, our approaches are exploratory and not ready for integration into deployed systems. However, this line of work studies self-supervised learning and therefore has the potential to scale and eventually contribute to systems that do interact with humans. In such cases, it is crucial to consider the usual safety and alignment considerations. However, beyond this, \care, offers insights into algorithmic approaches for controlling and moderating model behavior. Specifically, \care identifies a simple rotation of embedding space that corresponds to a change in the attribute of the data. In principle, this transformation could be used to "canonicalize" data, preventing the model from relying on certain attributes in decision-making. Additionally, controlled transformations of embeddings could be used to debias model responses and achieve desired variations in output. It is important to note that while our focus is on the core methodology, we do not explore these possibilities in this particular work.

\end{document}